\documentclass[journal,comsoc]{IEEEtran}
\usepackage{amsmath}
\usepackage{amssymb}
\usepackage{algorithm}
\usepackage{algorithmic}
\usepackage{graphicx}
\usepackage{diagbox}
\usepackage{subfigure}
\usepackage{caption}
\usepackage[numbers,sort&compress]{natbib}
\newtheorem{theorem}{Theorem}
\newtheorem{lemma}{Lemma}
\newtheorem{proof}{Proof}[section]
\newtheorem{assumption}{Assumption}
\newtheorem{definition}{Definition}
\newtheorem{remark}{Remark}

\allowdisplaybreaks[4]

\ifCLASSINFOpdf
\else
\fi
\begin{document}
%
\title{Convergence Analysis and System Design for Federated Learning over Wireless Networks}
%
%
%

\author{Shuo~Wan$^{1,2}$,
        Jiaxun~Lu$^{2}$,
        Pingyi~Fan{*}$^{1}$,~\IEEEmembership{Senior~member,~IEEE,}
        Yunfeng~Shao$^{2}$,
        Chenghui~Peng$^{3}$,
        and~Khaled~B.~letaief$^{4}$,~\IEEEmembership{Fellow,~IEEE}\\
        \small
        $^1$Department of Electronic Engineering, Tsinghua University, Beijing, P.R. China\\
        $^2$Huawei Noah’s Ark Lab\\
        $^3$Huawei Wireless Technology Lab\\
        $^4$Department of Electronic Engineering, Hong Kong University of Science and Technology, Hong Kong\\
        E-mail:
        wan-s17@mails.tsinghua.edu.cn,
        *fpy@mail.tsinghua.edu.cn,
        \{lujiaxun, shaoyunfeng, pengchenghui\}@huawei.com,
        eekhaled@ece.ust.hk}



\maketitle

\begin{abstract}
Federated learning (FL) has recently emerged as an important and promising learning scheme in IoT, enabling devices to jointly learn a model without sharing their raw data sets. However, as the training data in FL is not collected and stored centrally, FL training requires frequent model exchange, which is largely affected by the wireless communication network. Therein, limited bandwidth and random package loss restrict interactions in training. Meanwhile, the insufficient message synchronization among distributed clients could also affect FL convergence. In this paper, we analyze the convergence rate of FL training considering the joint impact of communication network and training settings. Further by considering the training costs in terms of time and power, the optimal scheduling problems for communication networks are formulated. The developed theoretical results can be used to assist the system parameter selections and explain the principle of how the wireless communication system could influence the distributed training process and network scheduling.
\end{abstract}

\begin{IEEEkeywords}
Federated learning, Network scheduling, Coupling design, Convergence analysis, Edge computing
\end{IEEEkeywords}

%
\IEEEpeerreviewmaketitle

\section{Introduction}
%
%
%
%
\IEEEPARstart{W}{ith} the emergence of data sets and the rapid growth of distributed computing, distributed learning has become a promising mode for deployment of large-scale machine learning \cite{bottou2018optimization}. Under such circumstances, due to data privacy and limited communication resources, it is very difficult to transmit all the raw data sets to a central server for learning, especially for applications with widely distributed intelligent clients. Thus, to implement machine learning with high efficiency, researchers started to focus on distributed learning schemes. In \cite{mcmahan2017communication}, the authors proposed the federated learning (FL) scheme for distributed data sets. Such a technique allows applications to collectively reap the benefits of shared models trained from the rich data while avoiding the need of central data collection.

However, implementation of FL could face constraints from insufficient communication network support, especially in a highly distributed system. Unlike conventional training with all resources on the cloud, FL training does not have direct access to all raw distributed data. Thus, it requires periodic exchange of model data over a wireless network among clients. Without sufficient message exchange, FL training could face severe degradation compared with centralized training. Therefore, in FL scheme, network scheduling would largely affect the convergence of training process.

To support FL with high efficiency, the network design must be linked with FL training to consider their coupling properties. Conventional network design mainly considers the communication efficiency without considering the computation algorithms. However, since convergence of FL can be largely affected by message exchange, how to support such a training algorithm with limited network resources would rise up as an important problem. Under such circumstances, the performance gain of FL from network scheduling would build a bridge between communication and computation. By proper analysis of such a bridge, a coupled network design can be achieved.


The problem of wireless network scheduling for FL training takes two steps to be solved. Firstly, a theoretical analysis of its convergence is required. Secondly, based on the convergence analysis, a network model should be set up for FL implementation and the settings therein will be optimized accordingly. Although convergence analysis for distributed learning has been widely studied, the analysis considering the impact of network settings has not been investigated with general and precise results. To the best of our knowledge, former works mainly consider settings of computation, such as the accuracy threshold, batch size, etc. However, as discussed above, communication will affect the model aggregation, which leads to a degraded convergence. Therefore, to support FL in network settings, the convergence analysis considering network support is an essential bridge. Then by jointly considering its tradeoff with costs in the communication network by model analysis, the optimal settings will be derived. Meanwhile, FL implementation mainly depends on the cooperation among widely distributed clients. Degradation in some participating clients also leads to a worse performance of the whole training process. Under such circumstances, the resource scheduling among clients is also required.

\subsection{Related works}
Recently some works studied the convergence of FL training with several bounds \cite{li2019convergence,haddadpour2019convergence,sahu2018convergence,wang2019adaptive,liu2020client}. Authors of \cite{haddadpour2019convergence} analyzed the general converging speed of FL. Considering the impact of training parameters, \cite{li2019convergence} proposed an analytic bound. However, the background communication was not considered systematically and the bound still requires refinement with more explicit physical meanings. The work in \cite{sahu2018convergence} tried to improve the classical FL scheme based on theoretical analysis. Meanwhile, \cite{wang2019adaptive} analyzed FL convergence considering the tradeoff between the local epoch and global epoch. Based on this work, \cite{liu2020client} further proposed a three-layer training scheme and discussed related control algorithms. However, in the proposed bounds, the tradeoff between training and communication is based on an additionally defined resource budget, which restricts the application range.

From the point of FL implementation, some works discussed the FL scheme and the scheduling algorithm therein \cite{wang2019adaptive,bhagoji2019analyzing,dinh2020federated,chen2020joint,luo2020hfel,ren2020accelerating,ng2020multi}. Authors of \cite{ng2020multi} proposed a multi-player game to study participants' reactions under
various incentive mechanisms in FL scenarios. \cite{bhagoji2019analyzing} studied the effects of malicious clients on FL. Considering the time slot division in the TDMA protocol, a control algorithm for FL over a wireless network was proposed in \cite{wang2019adaptive}. Authors of \cite{dinh2020federated} considered TDMA settings to jointly optimize the computation settings and time slot division in communication. However, the convergence analysis did not discuss the effects of communication network. As a result, the communication design becomes independent from training, without considering the tradeoff therein. Considering the package loss in communication, \cite{chen2020joint} analyzed the FL convergence rate and proposed a control algorithm. However, the convergence results are based on specific policies and can not provide more insights in general design. Besides, the local epoch, non-i.i.d. data set and partial participation were not considered, which restricts its generality. \cite{luo2020hfel,ren2020accelerating} also considered the optimization of communication settings. Although their communication models are solid, the convergence analysis still needs to be improved.

\subsection{Contributions}
Previous works have implemented a set of basic analysis models and scheduling algorithms for FL implementation. Therein, the convergence analysis of FL training still needs more observations of the background wireless network. Besides, part of the inequalities should be handled more tightly to reflect the trend of some important parameters. From the network design for FL perspective, one needs to consider the coupling properties between communication and computation. To resolve such problems, our contributions are summarized as follows.

\begin{enumerate}
\item The convergence rate analysis of FL training will be cast into a joint optimization problem of computing (training) and communication, in a more general setting with non-i.i.d. data sets, local training epochs, partial client participation, limited bandwidth, and package loss. By taking into account training settings as well as the impacts of the communication network on model aggregation, the derived convergence bound can be clearly divided into a computation part and a communication part with explicit physical meaning. Meanwhile, the tightness of the convergence bound is also improved. Thus, the impact of the intrinsic factors can be reflected with clearer physical meanings, fitting better in experiments.
\item Considering time and power consumption as the joint training cost, the general system cost for each training epoch is defined. By taking the convergence analysis as a bridge between FL training and wireless network, the overall setting for bandwidth and local training epoch are optimized with closed-form theoretical expressions. The result could fit the common knowledge of network and distributed learning, providing more insights of network settings for efficient FL. To the best of our knowledge, this is the first explicit result considering the tradeoff between computing (training) and communication in FL with closed-form principles for selecting hyper-parameters in a wireless network.
\item Considering the limitation of FL due to high level of distribution among clients, we propose network scheduling algorithms to enhance the cooperation among distributed clients. By adaptive resource scheduling, clients with varied capability and burden could achieve a better performance in synchronization. Given a specific network setting, the design could minimize the time and power cost accordingly. On the other hand, its derived costs will also affect the choice of the network hyper-parameters. Then by jointly considering the coupling factors, we set up an integrated design principle for FL implementation over wireless networks.
\end{enumerate}

\section{System model and Problem formulation}

The major focus is on a two-layer FL system composed of a central server and a set of $N$ distributed intelligent clients ($\mathbb{N}=\{1,2,...,N \}$). As shown in Fig. \ref{sys_model}, the server and clients are connected through a wireless network, where a cellular-based network is used as an example. Such a system could support various IoT applications such as environmental monitoring, health-care, etc.

\begin{figure}[tbp]
  \centering
  \includegraphics[width=3in]{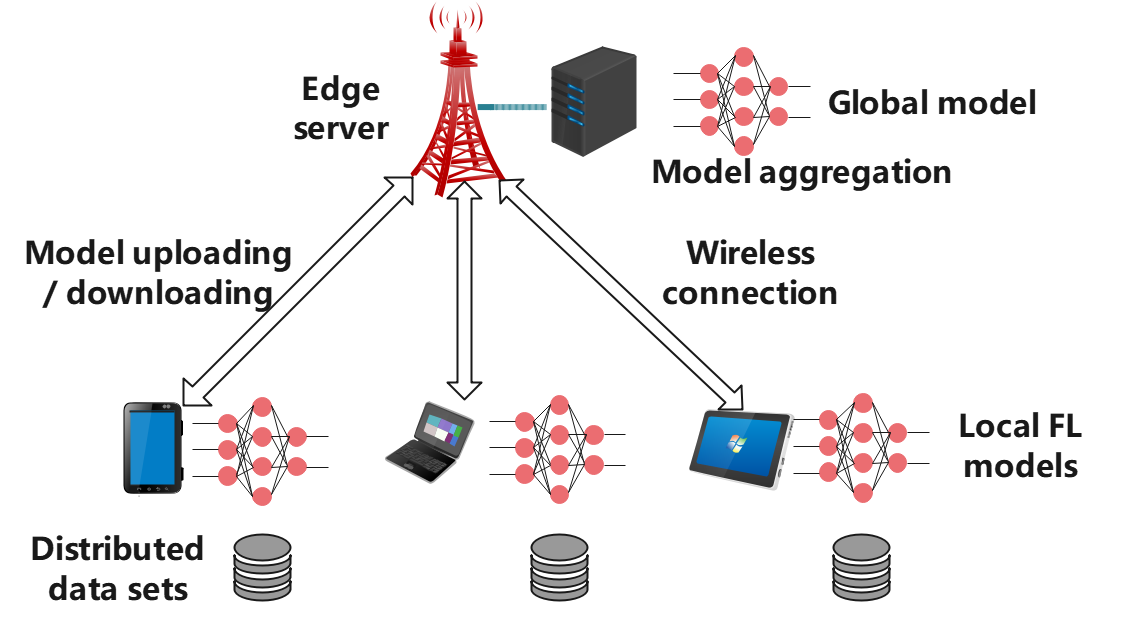}\\
  \caption{The structure of FL training over a wireless network.}\label{sys_model}
\end{figure}

\subsection{Federated learning process}\label{FL_learning}

In the FL system, the data set of client $j$ is $S_{j}=\left \{ (x_{1}^{(j)},y_{1}^{(j)}), (x_{2}^{(j)},y_{2}^{(j)}),...,(x_{D_{j}}^{(j)},y_{D_{j}}^{(j)}) \right \}$ with $D_{j}$ data samples. Before central aggregation, models are trained in distributed manner among clients. The loss function is
\begin{equation}
f_{j}(w) = \frac{1}{D_{j}} \sum_{i=1}^{D_{j}} l(w,(x_{i}^{(j)},y_{i}^{(j)})),
\end{equation}
where $l(w,(x_{i}^{(j)},y_{i}^{(j)}))$ is the loss function for model weight $w$ given $(x_{i}^{(j)},y_{i}^{(j)})$.
Denoting the weights of data set $S_{j}$ as $Q=\left \{ q_{1},q_{2},...,q_{N} \right \}$, the global loss function is
\begin{equation}\label{loss_f_g}
f(w)=\sum_{j=1}^{n} q_{j}f_{j}(w).
\end{equation}

\begin{table}[tbp]
\caption{The key notations}
\label{iid}
\centering
\begin{tabular}{|c|c|}
\hline
Notation& Definition \\
\hline
$N$ & The total number of clients participating in FL training.\\
\hline
$P_{t}$ & The current client set participating in model uploading.\\
\hline
$K$& Size of the randomly selected set $P_{t}$.\\
\hline
$D$ & The average size of distributed data sets among clients.\\
\hline
$E_{l}$ & Length of the local training epoch.  \\
\hline
$G_{\epsilon}$ & Total global epochs to reach loss $\epsilon$.\\
\hline
$\gamma$ & The package loss rate in wireless network.\\
\hline
$B$ & Total bandwidth for model uploading. \\
\hline
$\{a_{j}\}$ & Ratio of bandwidth allocation among clients in $P_{t}$.\\
\hline
$\{f_{j}\}$ & The processor frequency for local training.\\
\hline
$\{p_{j}^{0}\}$ & The uploading power for unit bandwidth.\\
\hline
$\{p_{j}\}$ & The model uploading power with $p_{j}=p_{j}^{0}a_{j}B$.\\
\hline
$r_{j}^{0}$ & The transmission rate with full bandwidth for client $j$.\\
\hline
$C_{u}$ & The expected model uploading cost in one global epoch.\\
\hline
$C_{u,0}$ & $E[\frac{1}{K}C_{u}]$: The expected uploading cost for unit $K$.\\
\hline
$C_{n}$ & The expected computation cost in one global epoch.\\
\hline
$C_{n,0}$ & $E[\frac{1}{E_{l}}C_{n}]$: The expected training cost for one local epoch.\\
\hline
$l_{0}$ & Weight of power cost in $C_{u}$ and $C_{n}$.\\
\hline
$\lambda$ & Metric of non-i.i.d. extent in distributed data sets.\\
\hline
$T_{d}$ & The time cost of model download.\\
\hline
\end{tabular}
\end{table}

The training process is done periodically by global epochs, each with $E_{l}$ local epochs. Let $t$ be an arbitrary discrete time slot for training, then $(t_{c}=\left \lfloor \frac{t}{E_{l}} \right \rfloor E_{l})$ should be the start slot of the current global epoch. At $t_{c}$, clients would receive renewed model weights $\bar{w}_{t_{c}}$ from the central server. As $t \neq t_{c}$, the distributed local training proceeds as
\begin{equation}\label{l_update}
w_{t_{c}+i+1}^{j} = w_{t_{c}+i}^{j} - \eta_{t_{c}+i} \bigtriangledown f_{j}(w_{t_{c}+i}^{j}),
\end{equation}
where $i=0,..,E_{l}-1$, $\eta_{t_{c}+i}$ is the learning rate and $\bigtriangledown f_{j}(w_{t_{c}+i}^{j})$ is the stochastic gradient in one local epoch. As local training is completed at $t=t_{c}+E_{l}$, the system will uniformly choose a client set $P_{t}$($|P_{t}|=K$,$K\leq N$) and aggregate their model weights ($w_{t}^{j}$, $j \in P_{t}$) at the server. Note that set $P_{t}$ is renewed at every round, so that each participating client gets the chance to upload its model. If $K=N$, all clients will be participators in each global epoch. The received weights are averaged as
\begin{equation}\label{gb_up_t}
\bar{w}_{t} = \frac{N}{K} \sum_{j \in P_{t}} q_{j}w_{t}^{j},
\end{equation}
where $E[\frac{N}{K} \sum_{j \in P_{t}} q_{j}]=1$.
In this process, the latency involves model uploading, central aggregation and model broadcast. FL training will then come to a new global epoch.

\subsection{Major problems and System Model}\label{basic_model}
FL highly depends on the communication network for model exchange in the training process.
Therein, more message exchanges in one round could reduce the total training epochs. However, due to limited bandwidth, it will also increase the data exchange time and power of each epoch. Meanwhile, scheduling of the highly distributed participators also affects the costs in model exchange.

In this case, crucial questions can be raised. How is FL training influenced by all the background network settings? What are the related principles to set the network parameters? How the network should be scheduled accordingly for implementation of a highly efficient FL training? In this paper, we try to answer these questions in theory and present related principles for the design and optimization of FL systems.

\subsubsection{Communication model}\label{com_model}

Let us consider FL over an orthogonal frequency division multiple access (OFDMA) network, which is widely adopted by 4G and 5G. As known, the uploading noise in OFMDA network rises with the allocated bandwidth. Thus, to obtain a rather stable transmission rate, the uploading power $p_{j}$ is set to be proportional to the allocated bandwidth. By denoting the total bandwidth as $B$, the transmission rate of client $j$ is given by
\begin{equation}\label{shannon}
r_{j}=a_{j}B{\rm log}_{2}\left ( 1+\frac{p_{j}h_{j}}{a_{j}BN_{0}} \right )
     =a_{j}B{\rm log}_{2}\left ( 1+\frac{p_{j}^{0}{h_{j}}}{N_{0}} \right )
\end{equation}
where $a_{j}$ ($j \in P_{t}$) is the proportion of bandwidth allocated to client $j$ with $\underset{j\in P_{t}}{\sum}a_{j} \leq 1$ and $p_{j}=p_{j}^{0}a_{j}B$. Therein, $p_{j}^{0}$ shall denote the transmission power for unit bandwidth and $h_{j}$ is the channel power gain of client $j$. The noise density in a wireless network is denoted as $N_{0}$. Then by defining $z_{m}$ as the required data size of transmitted model, the time cost of model uploading for client $j$ should be
\begin{equation}\label{up_time}
t_{u,j}=\frac{z_{m}}{r_{j}}=\frac{z_{m}}{a_{j}B{\rm log}_{2}\left ( 1+\frac{p_{j}^{0}{h_{j}}}{N_{0}} \right )}.
\end{equation}
Therein, the power consumption should be
\begin{equation}\label{up_power}
p_{u,j}=p_{j}t_{u,j}=p_{j}\frac{z_{m}}{r_{j}}.
\end{equation}
Note that the clients in set $P_{t}$ are selected uniformly in each global epoch, which would change with global epochs. Therefore, considering the expectation of the cost on $P_{t}$, the uploading cost is defined as
\begin{equation}\label{up_cost}
C_{u}=E_{P_{t}}[{\rm max}_{j \in P_{t}}\{ \frac{z_{m}}{a_{j}Br_{j}^{0}} \} + l_{0}\underset{j\in P_{t}}{\sum}p_{j}^{0}\frac{z_{m}}{r_{j}^{0}}],
\end{equation}
where $r_{j}^{0}={\rm log}_{2}\left ( 1+\frac{p_{j}{h_{j}}}{N_{0}} \right )$ is the transmission rate with unit bandwidth and $l_{0}$ is the weight of the power cost. Due to the synchronization requirement in model aggregation, the time cost is affected by the slowest client.


Wireless connection typically endures a random package loss due to the independent fading of wireless channels. Let us assume $K$ clients are uploading their model data simultaneously, and that $K_{\gamma}$ of them will be successfully received. For simplicity and considering the worst case, $K_{\gamma}\geq K(1-\gamma)$, where $\gamma < 1$. Note that the lost packages also take resources for transmission. Thus, the package loss does not influence $C_{u}$. Meanwhile, since the model aggregation here only involves $K_{\gamma}$ local models, FL training may take more global epochs to converge.

\subsubsection{Computation model}\label{computation_cost}
It is straightforward to see that computation latency in distributed training is proportional to the local epoch length $E_{l}$. The computation latency of client $j$ is
\begin{equation}\label{unit_computation}
t_{n,j}=E_{l}\frac{z_{n,j}}{f_{j}},
\end{equation}
where $f_{j}$ is the processor frequency and $z_{n,j}$ is the required processing cycles for one round of local training. The corresponding power consumption is
\begin{equation}\label{train_power}
p_{n,j}=E_{l}\kappa_{j}f_{j}^{2}z_{n,j},
\end{equation}
where $\kappa_{j}$ is a parameter depending on the specific processor on client $j$.

In fact, more data in training usually requires more processing cycles. This means that $z_{n,j}$ is proportional to $D_{j}$. Therefore, it is reasonable to assume that $z_{n,j}=\alpha_{0}D_{j}$ where $\alpha_{0}$ is an empirical parameter depending on training model and softwares. Meanwhile, since distributed training is an entirely local operation, the local computation cost should be the major focus in network scheduling. Thus, the average power consumption is taken instead of the summation of all training powers. Therefore, the computation cost of one global epoch is given by
\begin{equation}\label{com_cost}
C_{n}=E_{P_{t}}[{\rm max}_{j \in P_{t}}\{ E_{l}\frac{\alpha_{0}D_{j}}{f_{j}} \} + l_{0}\frac{1}{K}\underset{j\in P_{t}}{\sum}E_{l}\kappa_{j}f_{j}^{2}\alpha_{0}D_{j}].
\end{equation}


\subsection{Problem formulation}\label{problem_form}
We define $T_{d}$ as the extra time cost in the model broadcast by base station, which is rather fixed and determined by central resources. The package loss rate $\gamma$ is given as a background network parameter.
Then from the definitions of $C_{u}$ and $C_{n}$ in (\ref{up_cost}) and (\ref{com_cost}), the joint optimization problem for FL implementation over a wireless network can be expressed as follows.

\begin{align}
\min_{K,E_{l},\{a_{j}\},\{f_{j}\}} &G_{\epsilon}[C_{u}+C_{n}+T_{d}], \label{P_original}\\
s.t.\quad &G_{\epsilon} = G_{\epsilon}(E_{l},K,\gamma), \tag{\ref{P_original}{a}} \label{P_original_a}\\
&E_{l} \geq 1, 0\leq \gamma <1 \tag{\ref{P_original}{b}}, \label{P_original_b} \\
&1 \leq K \leq N, |P_{t}|=K \tag{\ref{P_original}{c}}, \label{P_original_c}\\
&\underset{j\in P_{t}}{\sum}a_{j} \leq 1, a_{j} > 0, j \in P_{t}, \tag{\ref{P_original}{d}} \label{P_original_d}\\
&f_{j}^{min}\leq f_{j} \leq f_{j}^{max}, j \in P_{t}. \tag{\ref{P_original}{e}} \label{P_original_e}
\end{align}

In the above formulation, the parameter $G_{\epsilon}$ is the number of global epochs taken by FL to reach a loss $\epsilon$, which will be given by convergence analysis. Constraint (\ref{P_original_a}) shows that $G_{\epsilon}$ is jointly affected by $K$, $\gamma$ and $E_{l}$. There exists an important tradeoff between $G_{\epsilon}$ and $C_{u}$. If $K$ increases with more clients uploading, $G_{\epsilon}$ will get smaller. However, from constraint (\ref{P_original_c}), more clients in $P_{t}$ leads to a smaller bandwidth for each client, which will in turn lead to a larger $C_{u}$. Therein, as the metric of FL training convergence, $G_{\epsilon}(E_{l},K,\gamma)$ is actually an important bridge between AI training and the communication network. In the subsequent section, we will derive the specific closed-form expression of $G_{\epsilon}$.

Constraints (\ref{P_original_b}), (\ref{P_original_c}), (\ref{P_original_d}) and (\ref{P_original_e}) are basic ranges in system settings. $E_{l}$ is the local training epoch, which stands for input from computation in training. $K$ is the size of $P_{t}$, representing the capability of model uploading provided by the wireless network. $\{a_{j}\}$ and $\{f_{j}\}$ are scheduling policies among clients, aimed at minimizing $(C_{u}+C_{n})$. Note that $K$ and $E_{l}$ are the hyper-parameters determining the total bandwidth division and local epoch. Given such settings, $\{a_{j}\}$ and $\{f_{j}\}$ are scheduled accordingly. By minimizing the cost for an arbitrary $P_{t}$, the expectation of the cost on random client selection can also be minimized. Meanwhile, on the contrary, the expected cost will in turn affect selection of $K$ and $E_{l}$, which can be reflected by theoretical results. By solving (\ref{P_original}), the integrated principles for the selection of the hyper-parameters and resource scheduling will be obtained.


\section{FL Convergence Analysis}
In the following, the analysis is based on the stochastic gradient descent (SGD) update. The training process is given in Section \ref{FL_learning} and the general non-i.i.d. data distribution is considered. In fact, the ideal i.i.d. data distribution is a special case therein. Beforehand, some common assumptions and the metric of non-i.i.d. in FL are first introduced. Then the expression of $G_{\epsilon}$ will be given and discussed in detail.

\subsection{Preparations}
\subsubsection{Assumptions on loss function}
The common assumptions of the L-smooth and $\mu$-Polyak-Lojasiewicz (PL) condition for the loss function are given as follows.

\begin{assumption}\label{L_smooth}
(L-smooth \cite{li2019convergence,haddadpour2019convergence}) The loss function $f(.)$ in FL training satisfies
\begin{align}
f(y) \leq f(x)+(y-x)^{T}\bigtriangledown f(x) + \frac{L}{2}\left \| y-x \right \|^{2}.
\end{align}
\end{assumption}

\begin{assumption}\label{P_L_condition}
($\mu$-P-L condition \cite{haddadpour2019convergence,karimi2016linear}) The loss function $f(.)$ in FL training satisfies a general extension of the $\mu$-strongly convex property, which is defined as
\begin{align}
|| \bigtriangledown f(x) ||^{2} \geq 2\mu[f(x)-f^{*}].
\end{align}
\end{assumption}
\begin{remark}
The P-L condition in Assumption \ref{P_L_condition} may not fit globally in classical neural networks like CNN. However, viewing $f^{*}$ as the local optimum, the assumption could still work. Therefore, it is reasonable to apply this in the convergence analysis. In our latter experiments with CNN, the rational of the assumption could be confirmed.
\end{remark}

\subsubsection{Non-i.i.d. data}
In practice, independent observed data at distributed clients typically diverge in probability distribution. For instance, some clients may have more data of cats while others may observe dogs. To proceed with the analysis, a commonly adopted measure is introduced to quantify such a property.

\begin{definition}\label{define_noniid}
\cite{haddadpour2019convergence,yin2018gradient}
Given $N$ clients with weights $\{\pi _{j}\} (\sum_{j=1}^{N}\pi_{j}=1)$ and $\{\bigtriangledown f_{j}(w)\}$ as their gradients, a measurement $\lambda$ for non-i.i.d. in data set is defined as
\begin{equation}\label{bound_non_iid}
  \frac{\sum_{j=1}^{N}\pi_{j}\left \| \bigtriangledown f_{j}(w) \right \|^{2}}{\left \| \sum_{j=1}^{N}\pi_{j}\bigtriangledown f_{j}(w) \right \|^{2}}=\wedge \leq \lambda.
\end{equation}
\end{definition}

\begin{remark}
It is known that $\lambda \geq 1$ by Jensen's inequality \cite{kuczma2009introduction}. Therein, $\lambda=1$ represents the i.i.d. condition. Parameter $\lambda$ reflects the non-i.i.d. extent of the stochastic gradients of $N$ clients.
\end{remark}

\subsection{Results of $G_{\epsilon}$}\label{conv_bound}

\begin{theorem}\label{estimate_G_epsilon}
Under Assumption \ref{L_smooth}), \ref{P_L_condition}) and Definition \ref{define_noniid}, if we choose the learning rate $\eta_{t}$ as $O(\frac{1}{t})$, $|P_{t}|=K$ and local epoch as $E_{l}$, $G_{\epsilon}$ is given by
\begin{align}\label{G_theorem}
    G_{\epsilon}= \frac{1}{\epsilon}\frac{4L^{2}G^{2}\lambda}{\mu^{2}}[ \frac{\lambda-1}{K(1-\gamma)}\frac{1}{2D} +\frac{(\lambda-1)E_{l}}{2C_{1}\phi_{0}}+\frac{1}{4}\frac{f_{0}}{E_{l}} ],
\end{align}
where $\gamma$ is the package loss rate in communication. Parameter $D=\frac{1}{N}\sum_{j=1}^{N}D_{j}$ is the average size of the training data and $G^{2}$ is defined as the gradient upper-bound. $C_{1}$ and $\sigma^{2}$ are constants related to the gradient variance in SGD and $(\phi_{0}-1)\propto \frac{\sigma^{2}}{\lambda}$.
Specific definitions of parameter $C_{1}$, $\sigma^{2}$ and $\phi_{0}$ can be found in Appendix \ref{appendix_pre} with the proof shown in Appendix \ref{proof_G}.
\end{theorem}


In (\ref{G_theorem}), $\frac{\lambda-1}{K(1-\gamma)}\frac{1}{2D}$ and $(\frac{(\lambda-1)E_{l}}{2C_{1}\phi_{0}}+\frac{1}{4}\frac{f_{0}}{E_{l}})$ are two key terms with clear physical meanings, separately representing the impacts resulting from communication and computation. By observing $\frac{\lambda-1}{K(1-\gamma)}\frac{1}{2D}$, we could see that $K$ and $1-\gamma$ compensate with each other to speed up the convergence in the form of product.
The communication term is also proportional to $(\lambda-1)$.
Thus, the extent of non-i.i.d. will influence the effects of the communication network. This can be observed from the experiments in Section \ref{experiments}.

The effects of computation are reflected by $(\frac{(\lambda-1)E_{l}}{2C_{1}\phi_{0}}+\frac{1}{4}\frac{f_{0}}{E_{l}})$ in (\ref{G_theorem}).
The optimal $E_{l}$ can be derived as $E_{l}^{*}=\sqrt{\frac{C_{1}\phi_{0}f_{0}}{2(\lambda-1)}}$. As $E_{l}$ increases, the local training becomes more sufficient. But increasing $E_{l}$ may also cause convergence failure due to the large diversity among the distributed data sets. Thus, a proper selection of $E_{l}$ must be balanced to speed up the training. For the i.i.d. setting with $\lambda=1$, $E_{l}$ can be arbitrarily large, which would be limited by the capability of local processors.

It is note that a larger $K$ naturally leads to a smaller allocated bandwidth. Thus $C_{u}$ in (\ref{up_cost}) will definitely get larger. Combined with Theorem \ref{estimate_G_epsilon}, the tradeoff between $G_{\epsilon}$ and $C_{u}$ can be easily observed.

\section{Design principle for FL}
In this section, the joint design principle for parameter selection and resource scheduling will considered.

\subsection{Sub-problems for hyper-parameters and scheduling policy}
As known, the client set $P_{t}$ for model uploading is renewed by uniform selection in each global epoch. Due to different conditions of the clients, the specific system cost in each global epoch may be varied. Thus, in the definition of $C_{u}$ and $C_{n}$ in (\ref{up_cost}) and (\ref{com_cost}), the expectation on $P_{t}$ is taken to get the expected cost for each global epoch.

As discussed in Section \ref{problem_form}, $K$ and $E_{l}$ are network hyper-parameters while $\{a_{j}\}$ and $\{f_{j}\}$ are scheduling policies based on a given setting. Thus, it is natural to see that finding solution of the hyper-parameters and scheduling policy should be decoupled. Before the decoupling, the expected costs $C_{u}$ and $C_{n}$ should first be transformed as a function of the hyper-parameters.

From the definition in (\ref{com_cost}), $C_{n}$ is actually proportional to $E_{l}$. Therein, setting $C_{n}=E_{l}C_{n,0}$, the unit cost $C_{n,0}$ can be directly defined as
\begin{equation}\label{u_Cn}
C_{n,0}=E_{P_{t}}[{\rm max}_{j \in P_{t}}\{ \frac{\alpha_{0}D_{j}}{f_{j}} \} + \frac{l_{0}}{K}\underset{j\in P_{t}}{\sum}\kappa_{j}f_{j}^{2}\alpha_{0}D_{j}].
\end{equation}

Considering $C_{u}$, as $K$ increases, more clients will be uploading their models simultaneously with less bandwidth for each client. Since the clients are chosen uniformly in $P_{t}$, then $E[\frac{1}{a_{j}}] \propto K$ holds for an arbitrary policy to allocate the bandwidth among $K$ clients in $P_{t}$. That is, the expected allocated bandwidth for each client is inversely proportional to $K$. Then by observing the definition of $C_{u}$ in (\ref{up_cost}), the uploading cost should be proportional to $K$ in the form of $C_{u}=KC_{u,0}$.

Considering the full participation case with $K=N$, the uploading cost for an arbitrary scheduling policy can be given as $({\rm max}_{j=1}^{N}\{ \frac{z_{m}}{a_{j}Br_{j}^{0}} \} + l_{0}\sum_{j=1}^{N}p_{j}^{0}\frac{z_{m}}{r_{j}^{0}})$. Then combined with $E[\frac{1}{a_{j}}] \propto K$, it is natural to see that the cost $E_{P_{t}}({\rm max}_{j\in P_{t}}\{ \frac{z_{m}}{a_{j}Br_{j}^{0}} \} + l_{0}\underset{j\in P_{t}}{\sum} \frac{p_{j}^{0}z_{m}}{r_{j}^{0}})$ statistically equals to $\frac{K}{N}({\rm max}_{j=1}^{N}\{ \frac{z_{m}}{a_{j}Br_{j}^{0}} \} + l_{0}\sum_{j=1}^{N}p_{j}^{0}\frac{z_{m}}{r_{j}^{0}})$ for an arbitrary scheduling policy of $\{a_{j}\}$. Thus, by dividing the full participation cost with $N$, $C_{u,0}$ could be derived as
\begin{equation}\label{u_Cu}
C_{u,0}=\frac{1}{N}[{\rm max}_{j=1}^{N}\{ \frac{z_{m}}{a_{j}Br_{j}^{0}} \} + l_{0}\sum_{j=1}^{N}p_{j}^{0}\frac{z_{m}}{r_{j}^{0}}].
\end{equation}

By considering the expected cost in the form ($KC_{0}+E_{l}C_{n,0}$) for a given scheduling policy, the sub-problem for the optimal hyper-parameters is defined as follows.

\begin{align}
\min_{K,E_{l}} &G_{\epsilon}[KC_{u,0}+E_{l}C_{n,0}+T_{d}], \label{sub1}\\
s.t.\quad &G_{\epsilon} = G_{\epsilon}(E_{l},K,\gamma), \tag{\ref{sub1}{a}} \label{sub1_a}\\
&E_{l} \geq 1, 0\leq \gamma <1 \tag{\ref{sub1}{b}}, \label{sub1_b} \\
&1 \leq K \leq N. \tag{\ref{sub1}{c}} \label{sub1_c}
\end{align}

For an arbitrary $P_{t}$, $K$ and $E_{l}$, the sub-problem for the optimal scheduling policy is defined as
\begin{align}
\min_{\{a_{j}\},\{f_{j}\}} & C_{u,t}+C_{n,t}, \label{sub2}\\
s.t.\quad &C_{u,t}={\rm max}_{j \in P_{t}}\{ \frac{z_{m}}{a_{j}Br_{j}^{0}} \} + l_{0}\underset{j\in P_{t}}{\sum}p_{j}^{0}\frac{z_{m}}{r_{j}^{0}},\tag{\ref{sub2}{a}} \label{sub2_a}\\
&C_{n,t}={\rm max}_{j \in P_{t}}\{ E_{l}\frac{\alpha_{0}D_{j}}{f_{j}} \} + \frac{l_{0}}{K}\underset{j\in P_{t}}{\sum}E_{l}\kappa_{j}f_{j}^{2}\alpha_{0}D_{j},\tag{\ref{sub2}{b}} \label{sub2_b}\\
&\underset{j\in P_{t}}{\sum}a_{j} \leq 1, a_{j} > 0, j \in P_{t}, \tag{\ref{sub2}{c}} \label{sub2_c}\\
&f_{j}^{min}\leq f_{j} \leq f_{j}^{max}, j \in P_{t}. \tag{\ref{sub2}{d}} \label{sub2_d}
\end{align}

The coupling properties between sub-problems (\ref{sub1}) and (\ref{sub2}) can be explained as follows.
Given $K$, $P_{t}$ and $E_{l}$, the optimal $\{a_{j}\}$ and $\{f_{j}\}$ can be derived by solving sub-problem (\ref{sub2}). That is, by taking the hyper-parameters as the input, it would give the optimal resource scheduling and the corresponding costs. In this sense, sub-problem (\ref{sub2}) actually represents the optimal policy to minimize $C_{u,0}$ and $C_{n,0}$. The sub-problem (\ref{sub1}) takes $C_{u,0}$ and $C_{n,0}$ as the input to get the optimal hyper-parameters. By representing the costs $C_{u}$ and $C_{n}$ as $KC_{u,0}$ and $E_{l}C_{n,0}$, the hyper-parameters would be optimized by considering tradeoff between $G_{\epsilon}$ and $C_{u}+C_{n}$. Therein, $G_{\epsilon}$ is given theoretically by Theorem \ref{estimate_G_epsilon}.

The joint optimal design for FL implementation is given by combining the two sub-problems. Sub-problem (\ref{sub2}) gives the scheduling policy for arbitrarily given parameters, while the expected costs given by such policy are taken as input to optimize the hyper-parameters in sub-problem (\ref{sub1}). Given $G_{\epsilon}$ as the bridge between training and communication, the hyper-parameters would consider the coupling effects from AI training and the wireless network with an optimal balance. By properly setting the hyper-parameters and scheduling policy therein, FL training with minimum cost in time and power can be achieved.

\subsection{Optimal hyper-parameters}\label{opt_KE}
By solving sub-problem (\ref{sub1}) combined with Theorem \ref{estimate_G_epsilon}, the principles for the optimal hyper-parameters can be derived.
\subsubsection{Bandwidth setting}
\begin{theorem}\label{K_result}
In FL training, the system cost related to $E_{l}$ and $K$ is given by $(KC_{u,0}+E_{l}C_{n,0}+T_{d})$. Under Theorem \ref{estimate_G_epsilon}, the optimal $K$ in (\ref{P_original}) for FL implementation is
\begin{align}\label{optimal_K}
K^{*}=\rho_{0}\sqrt{\frac{E_{l}}{D(1-\gamma)}} \sqrt{\frac{C_{n,0}+T_{d}/E_{l}}{C_{u,0}}},
\end{align}
where $\rho_{0}=\sqrt[4]{\frac{2(\lambda-1)C_{1}\phi_{0}}{f_{0}}}$ is a multiplier related to the background training process.
The proof is in Appendix \ref{proof_K_result}.
\end{theorem}

Viewing $\frac{T_{d}}{E_{l}}$ as the unit waiting time to begin local training, $T_{d}$ can be considered as part of the computation cost. Thus, $K^{*}$ is related to the ratio between $C_{n,0}$ and $C_{u,0}$.
As $C_{u,0}$ gets larger, a higher communication cost would decrease $K^{*}$. Meanwhile, a larger $C_{n,0}$ is matched with larger $K^{*}$, so that more sufficient training will not be wasted by insufficient communication. In this case, the design achieves a balance by comparing costs in communication and computation.

The tendency of $K^{*}$ in Theorem \ref{K_result} could fit common insights of FL training over wireless networks. By observing that $\sqrt{\frac{E_{l}}{D(1-\gamma)}}$, $K^{*}$ increases with $E_{l}$ and decreases with $1-\gamma$. In non-i.i.d. data sets, increasing $E_{l}$ results in a higher level of model divergence among distributed clients. It is then natural to expect a larger $K$. By considering the term $\sqrt{\frac{1}{1-\gamma}}$, a larger $\gamma$ means a lower rate of successful transmission in a wireless channel. Thus, $K$ should be increased for compensation.

The properties of the data sets and training algorithm also affect the network setting, reflected by $\rho_{0}$. By referring to Definition \ref{define_noniid}, $\lambda$ is the metric of non-i.i.d. case. As $\lambda$ increases, the local gradients will become more diverged, which leads to a larger $K$ for compensation. Meanwhile, $C_{1}$, $\phi_{0}$ and $f_{0}$ are related to the gradient variance and initial training loss, which also affect the need for model aggregation. The specific definitions of these parameters can be found in Appendix \ref{appendix_pre} and Appendix \ref{proof_G}.

As a coupled parameter in training and communication, $K^{*}$ is jointly determined by the training algorithm and wireless network. The closed-formed expression in Theorem \ref{K_result} reflects the specific influence of the cost ratio, $E_{l}$, $\gamma$, $D$ and $\lambda$. In system design, these factors should be jointly considered to adjust settings in wireless networks.


\subsubsection{$E_{l}$ setting}\label{setE_l}
The setting of $E_{l}$ is considered to minimize $G_{\epsilon}$, with $E_{l}^{*}=\sqrt{\frac{C_{1}\phi_{0}f_{0}}{2(\lambda-1)}}$ as the optimum value.
In FL training, the local epoch $E_{l}$ is originally introduced to lower down the frequency of communication \cite{mcmahan2017communication,li2019convergence}.
This is due to the fact that computation resources are typically more sufficient. Meanwhile, local training can be organized locally without complexity in interaction. At local processors, some more rounds of training may not increase much cost compared with the whole process of model uploading and downloading. Thus, in FL training, it is reasonable to set $E_{l}^{*}$ to minimize $G_{\epsilon}$.

Though $\phi_{0}$, $C_{1}$ and $f_{0}$ makes it hard to directly compute $E_{l}^{*}$, experiments that shall be provided later will show that $E_{l}^{*}$ does exit. Thus, assisted by the theory and estimation in experiments, $E_{l}^{*}$ can be obtained as one empirical parameter.

By jointly considering the results of $K^{*}$ and $E_{l}^{*}$ and the background loss rate $\gamma$, the integrated principles for choosing hyper-parameters can be obtained. Such design could ensure a faster training convergence with less power cost, which is meaningful for FL implementation.

\subsection{Scheduling policy}\label{opt_af}
The principles for optimizing the hyper-parameters have been discussed theoretically in Section \ref{opt_KE}. Given $K$, $E_{l}$ and a randomly selected client set $P_{t}$, the solution of sub-problem (\ref{sub2}) for optimal $\{a_{j}\}$ and $\{f_{j}\}$ will be discussed in this sub-section.

From (\ref{sub2}), (\ref{sub2_a}) and (\ref{sub2_b}), the optimization objective of sub-problem (\ref{sub2}) can be written as $[{\rm max}_{j \in P_{t}}(\frac{z_{m}}{a_{j}Br_{j}^{0}})+{\rm max}_{j \in P_{t}}(\frac{E_{l}\alpha_{0}D_{j}}{f_{j}})+l_{0}\underset{j\in P_{t}}{\sum}(\frac{z_{m}p_{j}^{0}}{r_{j}^{0}}+\frac{1}{K}E_{l}\alpha_{0}\kappa_{j}D_{j}f_{j}^{2})]$. Note that that in practice, once the local training is completed, a client may immediately upload the model without waiting for others. Thus, the computation time and uploading time can be combined together with the transformed objective as $[{\rm max}_{j \in P_{t}}(\frac{z_{m}}{a_{j}Br_{j}^{0}}+\frac{E_{l}\alpha_{0}D_{j}}{f_{j}})+l_{0}\underset{j\in P_{t}}{\sum}(\frac{z_{m}p_{j}^{0}}{r_{j}^{0}}+\frac{1}{K}E_{l}\alpha_{0}\kappa_{j}D_{j}f_{j}^{2})]$.
Then by defining parameter $H$ with $\frac{z_{m}}{a_{j}Br_{j}^{0}}+\frac{E_{l}\alpha_{0}D_{j}}{f_{j}} \leq H$ for arbitrary $j$ in set $P_{t}$, the sub-problem (\ref{sub2}) can be transformed as follows.

\begin{align}
\min_{\{a_{j}\},\{f_{j}\}} & l_{0}\underset{j\in P_{t}}{\sum}(\frac{z_{m}p_{j}^{0}}{r_{j}^{0}}+\frac{1}{K}E_{l}\alpha_{0}\kappa_{j}D_{j}f_{j}^{2}) +H, \label{sub3}\\
s.t.\quad & H \geq \frac{z_{m}}{a_{j}Br_{j}^{0}}+\frac{E_{l}\alpha_{0}D_{j}}{f_{j}}, \tag{\ref{sub3}{a}} \label{sub3_a}\\
&\underset{j\in P_{t}}{\sum}a_{j} \leq 1, a_{j} > 0, j \in P_{t}, \tag{\ref{sub3}{b}} \label{sub3_b}\\
&f_{j}^{min}\leq f_{j} \leq f_{j}^{max}, j \in P_{t}. \tag{\ref{sub3}{c}} \label{sub3_c}
\end{align}
It is straightforward to observe that (\ref{sub3}) is a convex optimization problem. The corresponding Lagrange function can be defined as
\begin{align}\label{L_function}
L=&l_{0}\underset{j\in P_{t}}{\sum}(\frac{z_{m}p_{j}^{0}}{r_{j}^{0}}+\frac{1}{K}E_{l}\alpha_{0}\kappa_{j}D_{j}f_{j}^{2})+H\notag\\&+R(\underset{j\in P_{t}}{\sum}a_{j}-1)+\underset{j\in P_{t}}{\sum}\beta_{j}(\frac{z_{m}}{a_{j}Br_{j}^{0}}+\frac{E_{l}\alpha_{0}D_{j}}{f_{j}}-H),
\end{align}
where $R$ and $\{\beta_{j}\}$ are dual variables.
By KKT conditions on (\ref{L_function}), the closed-form expression for $a_{j}$ can be obtained by the following theorem.
\begin{theorem}\label{t_aj}
Given the transmission power $p_{j}^{0}$, channel gain $h_{j}$ and background noise density $N_{0}$, then the bandwidth allocation is given by
\begin{equation}\label{a_j}
a_{j}=\frac{\sqrt{\frac{\kappa_{j}f_{j}^{3}}{r_{j}^{0}}}}{\underset{j\in P_{t}}{\sum} \sqrt{\frac{\kappa_{j}f_{j}^{3}}{r_{j}^{0}}}}.
\end{equation}
The proof is provided in Appendix \ref{proof_t_aj}.
\end{theorem}

Based on Theorem \ref{t_aj}, the scheduling policy in (\ref{sub2}) is proposed as follows.

\subsubsection{Optimal centralized solution}
From Theorem \ref{t_aj}, $a_{j}$ in (\ref{sub2}) can be substituted by $f_{j}$ and (\ref{sub2}) would be a convex problem for $f_{j}$. Then the optimal $\{f_{j}\}$ can be derived at the central server by applying convex optimization solvers and $\{a_{j}\}$ can be solved further.

Such a process could give an optimal solution by central scheduling. Therein, $\{a_{j}\}$ is the central bandwidth allocation, which is fitted for central scheduling. However, in the FL training mode, as the local processor frequency, it would be better for $f_{j}$ to be determined by clients locally. That is, client $j$ should directly schedule its own $f_{j}$ without counting on control from a central server. Besides, since set $P_{t}$ is renewed in each global epoch, frequently using optimization solvers to solve such a problem may cause much higher costs. Therefore, we propose a substitute method to solve for near-optimal $\{f_{j}\}$ in a distributed manner, which is rather simple and direct for FL engineering scenarios.

\subsubsection{Distributed scheduling policy}\label{dis_schedule}
Before optimization of $f_{j}$, the average processor frequency $\bar{f}=\frac{1}{N}\sum_{j=1}^{N} f_{j}$ is first considered as the basis. From definition of $C_{n}$ in (\ref{com_cost}), suppose that all clients take unified parameters $f_{j}=\bar{f}$, $D_{j}=D=\frac{1}{N}\sum_{j=1}^{N}D_{j}$ and $\kappa_{j}=\kappa=\frac{1}{N}\sum_{j=1}^{N}\kappa_{j}$, $\bar{f}$ is chosen to minimize the averaged computation cost $(E_{l}\frac{\alpha_{0}D}{\bar{f}} + l_{0}E_{l}\kappa \bar{f}^{2}\alpha_{0}D)$. By simple calculation of its stationary point, the optimal $\bar{f}$ is
\begin{equation}\label{ave_f}
\bar{f}=\sqrt[3]{\frac{1}{2l_{0}\kappa}}.
\end{equation}
Then due to the local data size $D_{j}$, $f_{j}$ can be given by
\begin{equation}\label{opt_fj}
f_{j}=\frac{D_{j}}{D}\bar{f}=\frac{D_{j}}{D}\sqrt[3]{\frac{1}{2l_{0}\kappa}}.
\end{equation}
Considering constraint $f_{j}^{min} \leq f_{j} \leq f_{j}^{max}$, the final $f_{j}$ can be derived by taking the intersection between (\ref{opt_fj}) and $[f_{j}^{min}, f_{j}^{max}]$.
In this process, the employed $\bar{f}$ in (\ref{ave_f}) considers the tradeoff between time and power in computation. With $\bar{f}$ as the basic reference, clients will schedule $f_{j}$ locally to achieve a balance of the distributed training time. Note that $\bar{f}$ is a rather fixed parameter, which can be optimized beforehand and reported to clients for reference. Thus, $f_{j}$ can be directly determined by clients in each round without counting on central control. Then, together with Theorem \ref{t_aj}, the optimal $a_{j}$ can be obtained.

\subsection{Integrated scheduling process}
Based on the principles for selecting the hyper-parameters in Section \ref{opt_KE} and the optimal scheduling policy in Section \ref{dis_schedule}, the proposed integrated scheduling process is shown in Fig. \ref{schedule_process}.

Before the learning process, the hyper-parameters $E_{l}$ and $K$ are first selected. $E_{l}$ is adjusted due to the extent of non-i.i.d. in data sets by referring to Theorem \ref{estimate_G_epsilon} and discussions in Section \ref{setE_l}. From the scheduling policy in Section \ref{dis_schedule}, the unit computation cost $C_{n,0}$ in (\ref{u_Cn}) for Theorem \ref{K_result} can be given as $C_{n,0}=\frac{\alpha_{0}D}{\bar{f}} + l_{0}\bar{f}^{2}\alpha_{0}\kappa D$. By combining Section \ref{dis_schedule} and Theorem \ref{t_aj}, the bandwidth allocation policy can be taken in (\ref{u_Cu}) to obtain $C_{u,0}$. Thus, by combining the cost from scheduling policy and the principles from the theoretical results, $K$ can be adjusted due to the cost ratio, package loss rate $\gamma$, average data size $D$ and $E_{l}$ by referring to Theorem \ref{K_result}.

\begin{figure}
  \centering
  \includegraphics[width=0.48\textwidth]{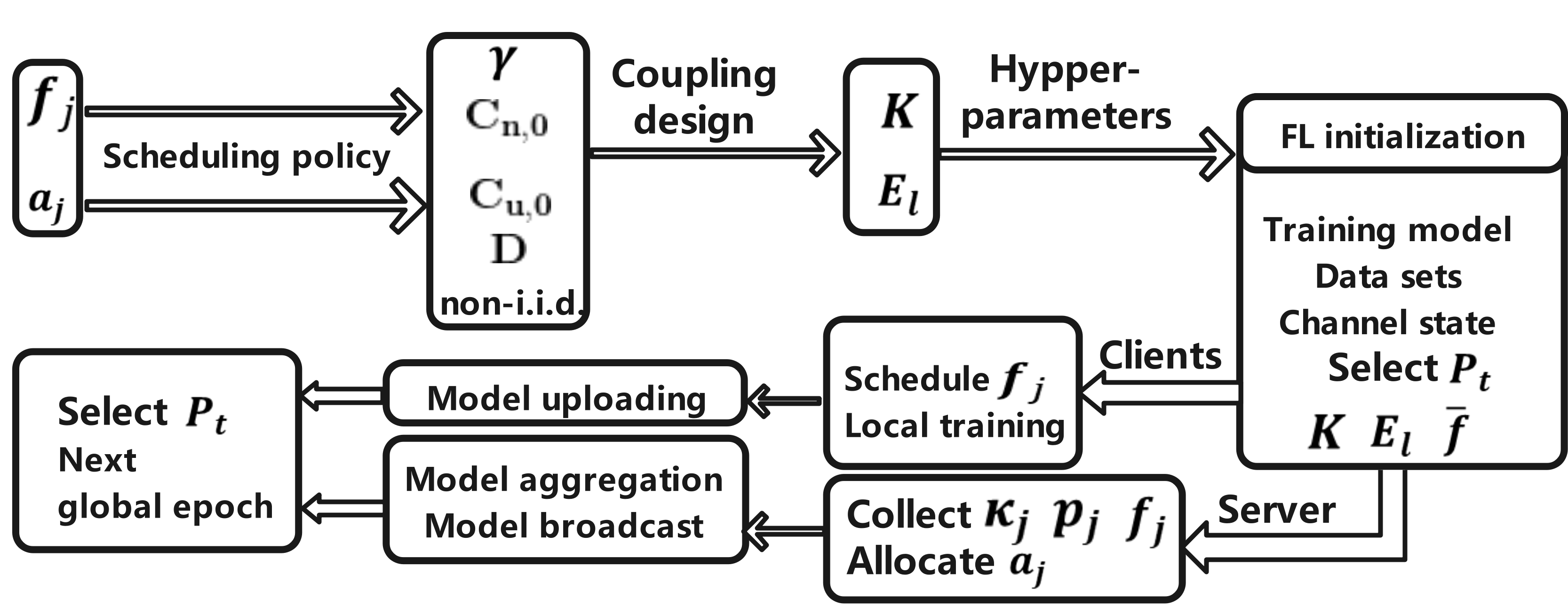}\\
  \caption{The scheduling process for FL in wireless networks.}\label{schedule_process}
\end{figure}

\begin{figure*}[tbp]
\centering  
\subfigure[Non-i.i.d. data set with $K$]{
\label{G_K}
\includegraphics[width=0.32\textwidth]{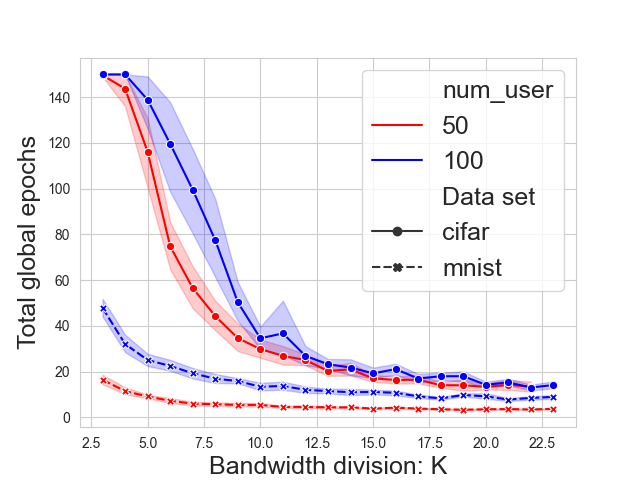}}
\subfigure[i.i.d. data set with $K$]{
\label{G_Kiid}
\includegraphics[width=0.32\textwidth]{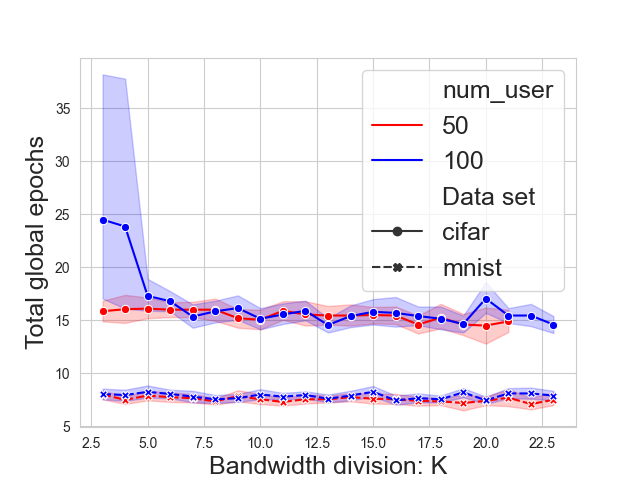}}
\subfigure[Non-i.i.d. data set with $E_{l}$]{
\label{G_E1}
\includegraphics[width=0.32\textwidth]{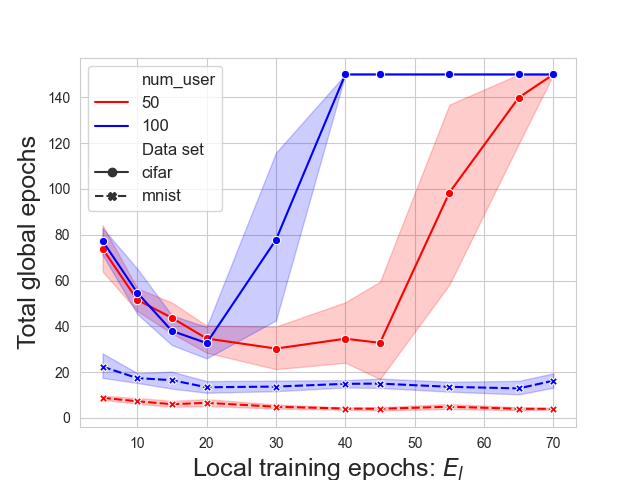}}
\caption{Trend of the total global epochs with respect to the bandwidth division and local epochs.}
\label{Fig.G}
\end{figure*}

FL begins with initialization of the training model, data sets and channel state of distributed clients. Note that some clients may endure a very bad channel or process very little data. To achieve an efficient scheduling, such clients will be first removed from the candidate client set. That is, the client set $P_{t}$ will exclude those in extremely bad conditions. Meanwhile, $\bar{f}$ in (\ref{ave_f}) will be initialized and broadcast for local scheduling of $f_{j}$. Then the random set $P_{t}$ will be selected before the training process begins. As the local training is conducted at distributed clients, the server will collect distributed states and allocate the bandwidth at the same time. Once the local training ends, the models will be uploaded and aggregated. Then $P_{t}$ will be renewed and the next global begins. Note that $\{a_{j}\}$ and $\{f_{j}\}$ are renewed for set $P_{t}$ in each global epoch. As shown in Fig. \ref{schedule_process}, the scheduling policy can be performed in parallel with the training process without much complexity.


\section{Experiments}\label{experiments}

\subsection{Basic test settings}
The database for training are selected as the classical MNIST and CIFAR10. The FL setting is the same as \cite{mcmahan2017communication} and the models are based on CNN. To simulate the distributed data sets in FL, all samples in the training data set are distributed among $N$ clients. As both MNIST and CIFAR10 are for multi-class classification, the non-i.i.d. data sets are realized by distributing the training data unevenly among the class labels. Meanwhile, the parameters for the communication channel are set as $B=20MHz$ and $N_{0}=5 \times 10^{-20}$. The random channel power gain is generated following an exponential distribution, where $h_{j}=g_{0}(\frac{d_{0}}{d})^{\theta} exp(1)$ with $g_{0}=10^{-4}$, $\theta=4$, $d_{0}=1$, and $d=200$.


\subsection{FL convergence}
Fig. \ref{G_K}, Fig. \ref{G_Kiid} and Fig. \ref{G_E1} show the trend of $G_{\epsilon}$ with respect to the parameters $K$ and $E_{l}$ in a non-i.i.d. setting. The target loss $\epsilon$ in the non-i.i.d. data sets is set as $1.88$ for CIFAR and $0.81$ for MNIST. For the i.i.d. case, it is set as $1.68$ for CIFAR and $0.2$ for MNIST. Total number of clients are set as $50$ and $100$. The size of the data sets in each client with $N=50$ is twice as much as that with $N=100$. Each point in the figure is obtained by running $20$ times with Monte Carlo simulation. The stripe around the curves represents the confidence interval and the maximum global training epoch is set as $150$ in each experiment.

As shown in Fig. \ref{G_K}, $G_{\epsilon}$ is approximately inversely proportional to $K$, consistent with the results in Theorem \ref{estimate_G_epsilon}. For $N=50$, $G_{\epsilon}$ gets smaller compared with $G_{\epsilon}$ for $N=100$. By observing the result in Theorem \ref{estimate_G_epsilon}, $N=50$ corresponds to a larger $D$, which is consistent with such a trend.

In the i.i.d. setting, Fig. \ref{G_Kiid} shows that $K$ has little effect on the convergence speed. By observing Theorem \ref{estimate_G_epsilon}, as $\lambda=1$, there is $\frac{\lambda-1}{K(1-\gamma)}\frac{1}{2D}=0$. Thus, the effects of $K$ and $D$ vanishes, which is consistent with the results listed in Fig. \ref{G_Kiid}. Thus, in the i.i.d. setting, we could set a smaller $K$ to save bandwidth.

Fig. \ref{G_E1} shows that $G_{\epsilon}$ has a minimum value with respect to the local epoch $E_{l}$. Such a tendency is consistent with the result in Theorem \ref{estimate_G_epsilon}. As MNIST is a rather simple data set for CNN, it is not sensitive to $E_{l}$. For the results in Fig. \ref{G_K} and Fig. \ref{G_Kiid}, $E_{l}$ is set as $20$ to get a relatively lower frequency of the model uploading.

\begin{figure}[!h]
\centering  
\subfigure[$G_{\epsilon}$ vs $\gamma$ ($K=10$, $E_{l}=20$)]{
\label{gamma1}
\includegraphics[width=0.23\textwidth]{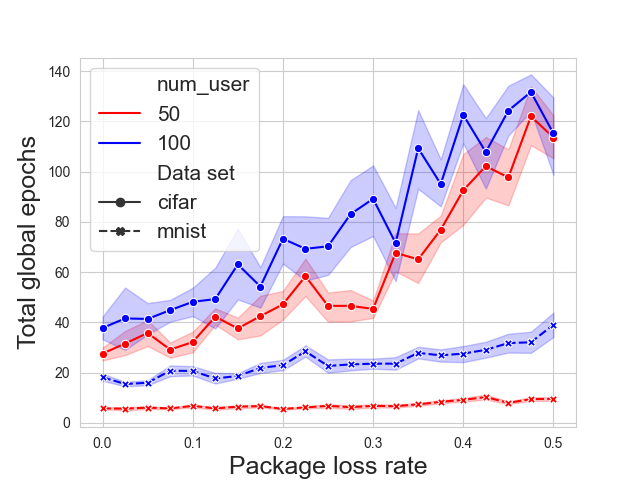}}
\subfigure[$G_{\epsilon}$ vs $K$ by referring to $\gamma$]{
\label{gamma2}
\includegraphics[width=0.23\textwidth]{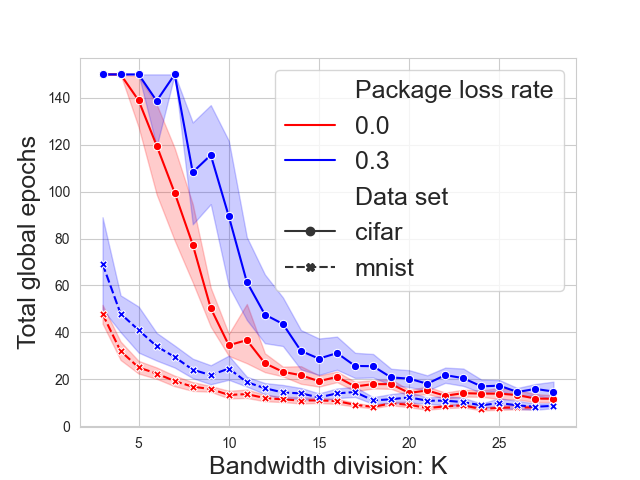}}
\caption{Impact of the package loss rate $\gamma$ on FL convergence.}
\label{Fig_gamma}
\end{figure}

Fig. \ref{Fig_gamma} shows the impact of the package loss rate $\gamma$. From Fig. \ref{gamma1}, as $\gamma$ increases from $0$ to $0.5$, FL convergence keeps slowing down. Since MNIST is a rather simple data set, it is shown to be less sensitive to $\gamma$. By comparing $G_{\epsilon}$ for $\gamma=0$ and $\gamma=0.3$ with respect to $K$, Fig. \ref{gamma2} shows that the effects of package loss can be compensated by increasing $K$ in proportion to $\gamma$. These results can be explained by the convergence analysis in Theorem \ref{estimate_G_epsilon}.

\subsection{Network scheduling}
\begin{figure}[!h]
\centering  
\subfigure[CIFAR data set]{
\label{optK1}
\includegraphics[width=0.23\textwidth]{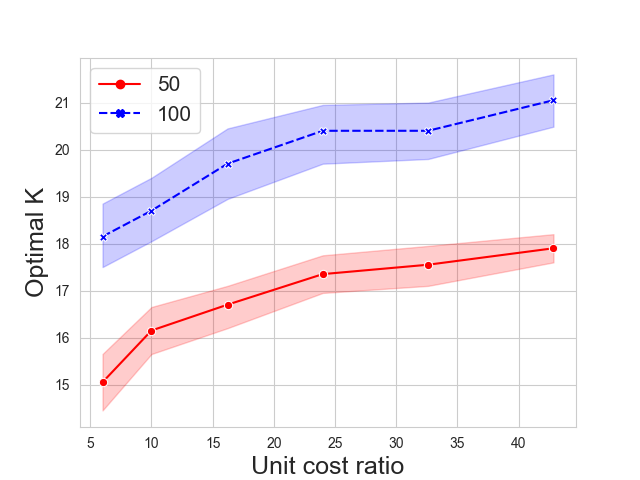}}
\subfigure[MNIST data set]{
\label{optK2}
\includegraphics[width=0.23\textwidth]{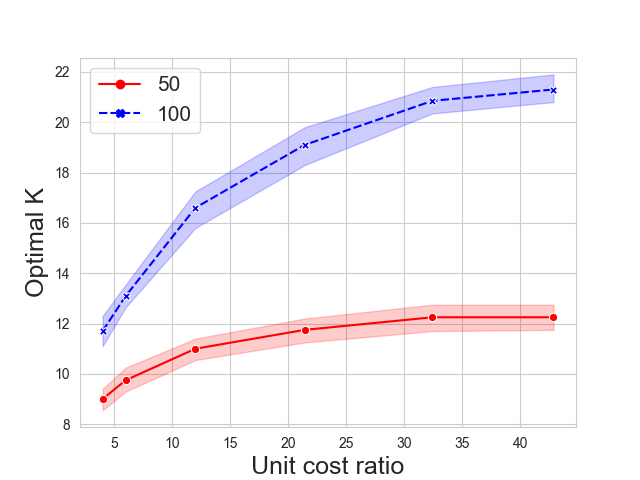}}
\caption{Optimal $K$ vs the cost ratio $\frac{E_{l}C_{n,0}}{C_{u,0}}$ in non-i.i.d. settings.}
\label{Fig.main}
\end{figure}
Based on the training loss for various $K$, we change the unit cost $C_{n,0}$ and $C_{u,0}$ to see the tendency of $K^{*}$ with the cost ratio $\frac{E_{l}C_{n,0}}{C_{u,0}}$ ($E_{l}=20$). In Fig. \ref{optK1} and Fig. \ref{optK2}, $K^{*}$ increases with the cost ratio. Meanwhile, both figures show that the curves of $K^{*}$ with $N=50$ are just below the curves with $N=100$. This is because each client in $N=50$ get more training data compared with that in $N=100$. Then by observing Theorem \ref{K_result}, $K^{*}$ will be reduced. Therefore, for the FL implementation over wireless networks, the system need to jointly consider the cost ratio and size of local data sets in order to have an efficient usage of limited bandwidth.

\begin{figure}[!h]
\centering  
\subfigure[$C_{u}+C_{n}$ vs $K$ ($E_{l}$=20)]{
\label{cost1}
\includegraphics[width=0.23\textwidth]{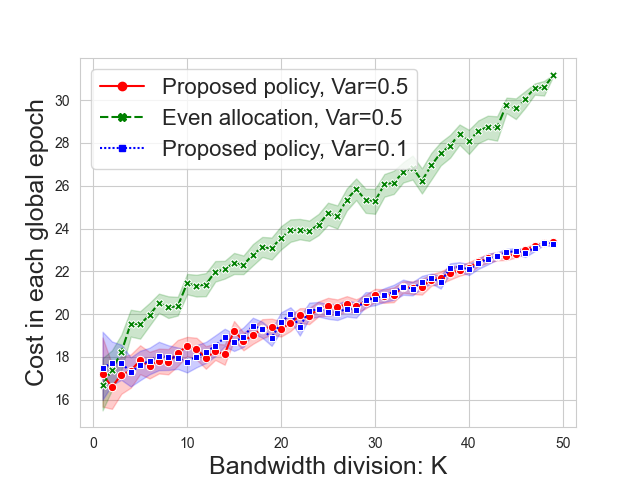}}
\subfigure[$C_{u}+C_{n}$ vs $E_{l}$ ($K$=10)]{
\label{cost2}
\includegraphics[width=0.23\textwidth]{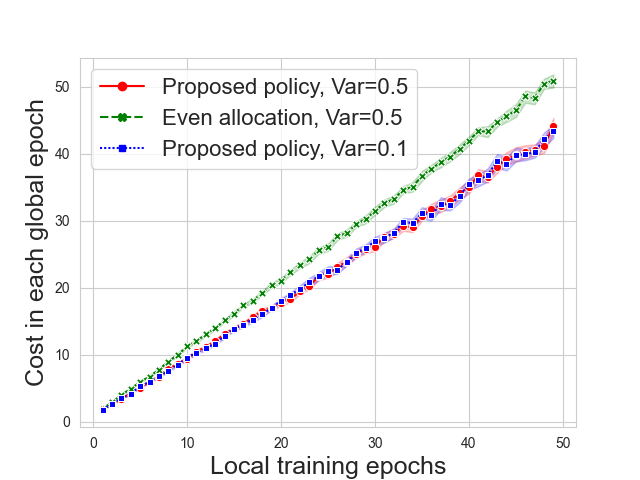}}
\caption{The system cost $C_{u}+C_{n}$ for different baseline settings.}
\label{Fig_cost}
\end{figure}
The system cost $C_{u}+C_{n}$ in each global epoch is shown in Fig. \ref{Fig_cost}. The curves are derived by the Monte Carlo method with independent simulations of $50$ epochs. Therein, the parameters for computation are set as $z_{m}=3\times 10^{4}$, $\alpha_{0}=5\times 10^{5}$ and $l_{0}=1$. Besides, $Var$ is selected to represent variations of the parameter settings among all clients with $0\leq Var <1$. Therein, $\kappa_{j}$, $p_{j}^{0}$, $D_{j}$ are derived uniformly in the range $[\bar{\kappa}(1-Var), \bar{\kappa}(1+Var)]$, $[\bar{p}^{0}(1-Var), \bar{p}^{0}(1+Var)]$ and $[D(1-Var), D(1+Var)]$ with $\bar{\kappa}=5\times 10^{-27}$, $\bar{p}^{0}=4\times 10^{-7}$ and $D=500$.

By observing the curves with respect to $K$ and $E_{l}$ in Fig. \ref{cost1} and Fig. \ref{cost2}, the former settings $C_{u}=KC_{u,0}$ and $C_{n}=E_{l}C_{n,0}$ can be directly validated. The evenly allocation policy in Fig. \ref{Fig_cost} refers to $a_{j}=\frac{1}{K}$ and $f_{j}=\bar{f}$. Compared with such a policy, the proposed network scheduling algorithm achieves a lower cost with less $C_{u,0}$ and $C_{n,0}$. Besides, by comparing the curves with $Var=0.1$ and $Var=0.5$, it is straightforward to see that the variation among clients can be handled well by the proposed policy without significant degradation in costs.

\section{Conclusion}
In this paper, we investigated the implementation of FL training over wireless networks. Considering a general FL training process on non-i.i.d. data sets, we jointly considered the coupling problem between communication and computing (training) in FL and presented some theoretical results in principle on the FL design. Therein, the inequalities were handled more elegantly to derive bounds with clearer physical meaning. Viewing latency and power consumption as a joint system cost, we discovered the tradeoff between FL convergence and system cost. The integrated design principles for FL implementation were also proposed in closed-form. Simulations on MNIST and CIFAR10 have demonstrated that the proposed convergence analysis and parameter selection principles have a good fit with the real recorded training loss. In particular, the work clearly pointed out how the wireless network and training settings would jointly influence FL convergence, which is meaningful for FL implementation in a wide range.



%

\appendices
\section{Preparations}\label{appendix_pre}
The standard FL discussed in Section \ref{FL_learning} chooses clients uniformly and averages the collected $w_{t}^{j}$ by weight $q_{j}$. In this paper, for simplicity, an equivalent process for model aggregation is adopted similarly as in \cite{li2019convergence,haddadpour2019convergence}, where clients are randomly selected in $P_{t}$ by probability $q_{j}$ and the global model is updated by
\begin{equation}\label{gb_up1}
\bar{w}_{t} = \frac{1}{K} \sum_{j \in P_{t}} w_{t}^{j}.
\end{equation}
This training process considers weight $q_{j}$ as the sampling probability in $P_{t}$ for client $j$ instead of the aggregation weight, which achieves an equivalent training as standard FL.
\subsection{Shorthand notations}
For simplicity, some shorthand notations are defined as follows.

$w_{t}^{j}$: The local model weight in client $j$ at time $t$.

$\widetilde{g}_{t}^{j}$: The local stochastic gradient of client $j$ at time $t$ due to random data $\xi$.

$D$: Average local data size among clients.

$g_{t}^{j}$: The expected gradient of client $j$ at time $t$. Note that $E_{\xi}(\widetilde{g}_{t}^{j})=g_{t}^{j}$.

$P_{t,\gamma}$: Set of clients successfully uploading models without package loss in set $P_{t}$. Its size is denoted as $|P_{t,\gamma}|=K_{\gamma}$.

$E[.]$ is the general expectation involving $E_{\xi}$ and $E_{P_{t,\gamma}}$. In the following proof, without specific explanation, $E[.]$ will be taken to denote expectations for short.

Based on these notations, some related terms of the global model and gradients are defined as follows.

Global aggregated model weight:
\begin{align}\label{gather_w}
\overline{w}_{t}=\frac{1}{K_{\gamma}}\sum_{j \in P_{t,\gamma}} w_{t}^{j}.
\end{align}

Global aggregated SGD gradient:
\begin{align}\label{gather_Sgd}
\widetilde{g}_{t}=\frac{1}{K_{\gamma}}\sum_{j \in P_{t,\gamma}} \widetilde{g}_{t}^{j}.
\end{align}

Expectation of global aggregated gradient:
\begin{align}\label{gather_gd}
\overline{g}_{t}=\frac{1}{K_{\gamma}}\sum_{j \in P_{t,\gamma}} g_{t}^{j}.
\end{align}

Update of aggregated global model:
\begin{align}\label{update_w}
\overline{w}_{t+1}=\overline{w}_{t}-\eta_{t}\widetilde{g}_{t}.
\end{align}

Note that $t$ in (\ref{gather_w}) (\ref{gather_Sgd}) (\ref{gather_gd}) can be an arbitrary time slot, not necessarily to be $t=t_{c}=\left \lfloor \frac{t}{E_{l}} \right \rfloor E_{l}$. The aggregation terms are derived from virtually aggregated values without affecting the local training, while the real aggregation with feedback only occurs at $t_{c}$.

\subsection{Assumptions of bounds on training gradients}
Some assumptions on the SGD gradient are given by referring to \cite{li2019convergence,haddadpour2019convergence}.
\begin{assumption}\label{gradient_variance}
The stochastic gradient on the local data set suffers a variance upper-bounded by
\begin{align}
E[ || \widetilde{g}_{t}^{j}-g_{t}^{j} ||^{2} ]  \leq \frac{C_{1} || g_{t}^{j} ||^{2}+\sigma^{2}}{D}l_{j},
\end{align}
where $C_{1}$ is the stochastic coefficient of the variation of gradients and $\sigma^{2}$ is the variance of noise in sampling. The gradient variance is inversely proportional to the local data size $D_{j}=\frac{D}{l_{j}}$, where $D$ is the average data size and $l_{j}$ is the specific ratio for client $j$. Taking client weight $q_{j}$ as $q_{j} \propto D_{j}$, there is
\begin{align}\label{re_weight}
\sum_{j=1}^{n}q_{j}l_{j}=1.
\end{align}
\end{assumption}

\begin{assumption}\label{gradient_bound}
Given $\{\pi_{j}\}$ as general client weights with $\sum_{j=1}^{n}\pi_{j}=1$ and $\lambda$ in Definition \ref{define_noniid},
the weighted gradients are supposed to be upper-bounded by
\begin{align}
C_{1}|| \sum_{j=1}^{N}\pi_{j}g_{k}^{j} ||^{2}+\frac{\sigma^{2}}{\lambda} \leq G^{2},
\end{align}
where $\sum_{j=1}^{N}\pi_{j}g_{k}^{j}$ is a weighted aggregation of the local model.
\end{assumption}
Under Assumption \ref{gradient_bound}, $C_{1}|| \sum_{j=1}^{N}\pi_{j}g_{k}^{j} ||^{2}$ is upper-bounded by
\begin{equation}\label{phi0}
C_{1}|| \sum_{j=1}^{N}\pi_{j}g_{k}^{j} ||^{2} \leq \frac{G^{2}}{\phi_{0}},
\end{equation}
where $\phi_{0}$ is a parameter related to the ratio between $G^{2}$ and the weighted gradients, which is affected by $\sigma^{2}$.

\subsection{A lemma for summation in random set $P_{t,\gamma}$}
\begin{lemma}\label{p_t_gamma}
Given $P_{t}$ with $|P_{t}|=K$ as set of clients for model aggregation, $P_{t,\gamma}$ with $|P_{t,\gamma}|=K_{\gamma}$ is the set of the actually received model by central server. Suppose that the package loss rate is i.i.d. in a wireless network with the uniform worst case $K_{\gamma}\geq K(1-\gamma)$, then the summation in random set $P_{t,\gamma}$ satisfies the following equations:
\begin{equation}\label{lemma_summation}
E_{P_{t,\gamma}}[\frac{1}{K_{\gamma}}\underset{j\in P_{t,\gamma}}{\sum}x_{j}] = \sum_{j=1}^{N}q_{j}x_{j},
\end{equation}
\begin{equation}\label{lemma_summation_2}
E_{P_{t,\gamma}}[\frac{1}{K_{\gamma}^{2}}\underset{j\in P_{t,\gamma}}{\sum}x_{j}] \leq \frac{1}{K(1-\gamma)} \sum_{j=1}^{N}q_{j}x_{j},
\end{equation}
where $q_{j}$ is the probability of choosing client $j$ in $P_{t}$ and $x_{j}$ denotes an arbitrary random variable from client $j$.
\end{lemma}
\begin{proof}
Setting $f_{0}(K_{\gamma})$ to be $\frac{1}{K_{\gamma}}$ or $\frac{1}{K_{\gamma}^{2}}$, then the summation can be deduced as follows.
\begin{align}\label{f2}
   E_{P_{t,\gamma}}[f_{0}(K_{\gamma}) \underset{j\in P_{t,\gamma}}{\sum}x_{j}] &=E_{P_{t}}E_{\gamma}[f_{0}(K^{\gamma}) \underset{j\in P_{t}^{\gamma}}{\sum}x_{j}]\notag\\
   &\overset{\textcircled{1}}{=} E_{\gamma}E_{P_{t}}[f_{0}(K^{\gamma}) \underset{j\in P_{t}^{\gamma}}{\sum}x_{j}]\notag\\
   &=E_{\gamma}[f_{0}(K^{\gamma})E_{P_{t}}[\underset{j\in P_{t}^{\gamma}}{\sum}x_{j}]]\notag\\
   &=E_{\gamma}[f_{0}(K^{\gamma}) \sum_{l=1}^{K^{\gamma}}E_{P_{t}}[x_{j_{l}}]]\notag\\
   &\overset{\textcircled{2}}{=}E_{\gamma}[f_{0}(K^{\gamma})K^{\gamma}\sum_{j=1}^{N}q_{j}x_{j}]\notag\\
   &=\sum_{j=1}^{N}q_{j}x_{j}E_{\gamma}[f_{0}(K^{\gamma})K^{\gamma}],
\end{align}
where $\textcircled{1}$ is due to the assumption that $\gamma$ is i.i.d. among the clients. By focusing on the worst case $K_{\gamma}\geq K(1-\gamma)$, such an assumption is reasonable. $\textcircled{2}$ is due to the assumption that client $j$ is selected by probability $q_{j}$. By taking in $f_{0}(K_{\gamma})=\frac{1}{K_{\gamma}}$ or $f_{0}(K_{\gamma})=\frac{1}{K_{\gamma}^{2}}$, the proof is completed.
\end{proof}

\section{Lemma \ref{g_t}}\label{proof_g_t}
\begin{lemma}\label{g_t}
Suppose that clients are selected in $P_{t}$ by probability $\{q_{j}\}$, under Assumption \ref{gradient_variance}, Definition \ref{define_noniid}, (\ref{lemma_summation_2}), (\ref{gather_Sgd}) and (\ref{gather_gd}), then $E|| \widetilde{g}_{t} ||^{2}$ is upper-bounded by
\begin{align}
    &E[ || \widetilde{g}_{t} ||^{2}  ] \leq \notag\\
    &\frac{\lambda C_{1}}{DK(1-\gamma)} \sum_{j=1}^{N}|| q_{j}l_{j}g_{t}^{j} ||^{2}
    +\frac{1}{K(1-\gamma)}\frac{\sigma^{2}}{D}+\lambda \sum_{j=1}^{N}|| q_{j}g_{t}^{j} ||^{2}.
\end{align}
\end{lemma}

{\textbf{Proof: }}

It is natural to see that $\bar{g}_{t}=E[\widetilde{g}_{t}]$. Due to the basic properties of expectation, $E[|| \widetilde{g}_{t} ||^{2}]$ can be transformed as
\begin{align}\label{g_t_1}
    E[|| \widetilde{g}_{t} ||^{2}]&=E[|| \widetilde{g}_{t}-E[\widetilde{g}_{t}] ||^{2}]+|| E[\widetilde{g}_{t}] ||^{2}\notag\\
    &=E[|| \widetilde{g}_{t}-\bar{g}_{t} ||^{2}]+|| \bar{g}_{t} ||^2.
\end{align}
$|| \bar{g}_{t} ||^2$ in (\ref{g_t_1}) can be upper-bounded as follows due to Jensen's inequality.
\begin{align}
|| \bar{g}_{t} ||^2 = || \frac{1}{K_{\gamma}}\underset{j \in P_{t,\gamma}}{\sum} g_{t}^{j} ||^{2} \leq \frac{1}{K_{\gamma}} \underset{j \in P_{t,\gamma}}{\sum} || g_{t}^{j} ||^{2}.
\end{align}

Under Assumption \ref{gradient_variance}, considering expectation $E_{\xi}[.]$ on the data sets, there is the following upper bound on $E[|| \widetilde{g}_{t}-\bar{g}_{t} ||^{2}]$ in (\ref{g_t_1}).
\begin{align}
    E_{\xi }[\left \| \widetilde{g}_{t}-\bar{g}_{t} \right \|^{2}]
    &= E_{\xi }[|| \frac{1}{K_{\gamma}} \underset{j \in P_{t,\gamma}}{\sum} \widetilde{g}_{t}^{j}- \frac{1}{K_{\gamma}} \underset{j \in P_{t,\gamma}}{\sum} g_{t}^{j} ||^{2}]\notag\\
    &=\frac{1}{K_{\gamma}^{2}} E_{\xi }[\underset{j \in P_{t,\gamma}}{\sum} \left \| \widetilde{g}_{t}^{j}-g_{t}^{j} \right \|^{2} \notag\\& + \underset{i \neq j}{\sum}<\widetilde{g}_{t}^{i}-g_{t}^{i} , \widetilde{g}_{t}^{j}-g_{t}^{j}> ]\notag\\
    &\overset{\textcircled{1}}{=}E_{\xi }[\frac{1}{K_{\gamma }^{2}}\underset{j \in P_{t,\gamma}}{\sum} || \widetilde{g}_{t}^{j}-g_{t}^{j} ||^{2}]\notag\\
    &\leq \frac{1}{K_{\gamma }^{2}}\underset{j \in P_{t,\gamma}}{\sum}l_{j}( \frac{C_{1}}{D}||  g_{t}^{j} ||^{2}+\frac{\sigma^{2}}{D}),
\end{align}
where $\textcircled{1}$ is from the common sense that the random local data sampling is independent among clients. That is, the error $\widetilde{g}_{t}^{i}-g_{t}^{i}$ resulting from local sampling in client $i$ is independent from $\widetilde{g}_{t}^{j}-g_{t}^{j}$ in client $j$. Note that the gradients themselves are not independent among clients, though their errors are in fact independent. Then from (\ref{re_weight}) and Lemma \ref{p_t_gamma}, there is
\begin{align}\label{g_bd}
    &E [ || \widetilde{g}_{t} ||^{2} ]\leq \notag\\
    & E_{P_{t,\gamma}}[ \frac{1}{K_{\gamma }^{2}} \underset{j \in P_{t,\gamma}}{\sum} \frac{C_{1}}{D}||  g_{t}^{j} ||^{2}+\frac{1}{K_{\gamma}} \underset{j \in P_{t,\gamma}}{\sum} || g_{t}^{j} ||^{2} +\frac{1}{K_{\gamma}}\frac{\sigma^{2}}{D} ] \leq \notag\\
    &\frac{C_{1}}{DK(1-\gamma)} \sum_{j=1}^{N}q_{j}l_{j}|| g_{t}^{j} ||^{2}
    +\frac{1}{K(1-\gamma)}\frac{\sigma^{2}}{D}+\sum_{j=1}^{N}q_{j}|| g_{t}^{j} ||^{2}.
\end{align}
Taking in metric $\lambda$ from Definition \ref{define_noniid}, the proof is completed.

\section{Lemma \ref{cross_bd}}\label{proof_cross}
\begin{lemma}\label{cross_bd}
Under Assumption \ref{gradient_variance}), \ref{L_smooth}), Definition \ref{define_noniid}, (\ref{gather_gd}), (\ref{gather_Sgd}), (\ref{gather_w}) and (\ref{update_w}), suppose clients are selected associated with weight $\{q_{j}\}$ and $\eta_{t}$ diminishing by $O(\frac{1}{t})$, then $-\eta_{t}E\left [ <\bigtriangledown f(\bar{w_{t}}), \widetilde{g}_{t}> \right ]$ is upper-bounded as follows.
\begin{align}\label{cross_bound}
    &-\eta_{t}E\left [ <\bigtriangledown f(\bar{w_{t}}), \widetilde{g}_{t}> \right ]\leq
    -\frac{1}{2}|| \bigtriangledown f(\bar{w_{t}}) ||^{2}-\frac{1}{2}|| \sum_{j=1}^{N}q_{j}g_{t}^{j} ||^{2} \notag\\
    &+\eta_{t}^{2}L^{2}\lambda\frac{\lambda-1}{K(1-\gamma)}\frac{C_{1}}{2D}\sum_{k=t_{c}+1}^{t_{c}+E_{l}}||\sum_{j=1}^{N}q_{j}l_{j} g_{k}^{j} ||^{2}+ \notag\\
    &\frac{E_{l}\eta_{t}^{2}L^{2}\sigma^{2}}{2D}\frac{\lambda-1}{K(1-\gamma)}
    +\frac{\lambda-1}{2}E_{l}\lambda L^{2} \eta_{t}^{2}\sum_{k=t_{c}+1}^{t_{c}+E_{l}}||\sum_{j=1}^{N} q_{j} g_{k}^{j} ||^{2}.
\end{align}
\end{lemma}

{\textbf{Proof: }}

From the basic definitions (\ref{gather_gd}), (\ref{gather_Sgd}), (\ref{gather_w}) and (\ref{update_w}), $-E [ <\bigtriangledown f(\bar{w_{t}}), \widetilde{g}_{t}>  ]$ can be transformed as follows.
\begin{align}\label{et_cross}
    &-E [ <\bigtriangledown f(\bar{w_{t}}), \widetilde{g}_{t}>  ]
    =-<\bigtriangledown f(\bar{w_{t}}) , E[\frac{1}{K^{\gamma}} \underset{j \in P_{t}^{\gamma}}{\sum} \widetilde{g}_{t}^{j} ]>\notag\\
    &\overset{\textcircled{1}}{=} -<\bigtriangledown f(\bar{w_{t}}) , \sum_{j=1}^{N}q_{j}g_{t}^{j}>
    \overset{\textcircled{2}}{=} -\frac{1}{2}|| \bigtriangledown f(\bar{w_{t}}) ||^{2} + \notag\\&\frac{1}{2}|| \bigtriangledown f(\bar{w_{t}})-\sum_{j=1}^{N}q_{j}g_{t}^{j} ||^{2} -\frac{1}{2}|| \sum_{j=1}^{N}q_{j}g_{t}^{j} ||^{2} \notag\\&= -\frac{1}{2}|| \bigtriangledown f(\bar{w_{t}}) ||^{2} + \frac{1}{2}|| \sum_{j=1}^{N} q_{j}(\bigtriangledown f_{j}(\bar{w_{t}})-\bigtriangledown f_{j}(w_{t}^{j})) ||^{2}
    \notag\\&-\frac{1}{2}|| \sum_{j=1}^{N}q_{j}g_{t}^{j} ||^{2}
    \overset{\textcircled{3}}{\leq}-\frac{1}{2}[|| \bigtriangledown f(\bar{w_{t}}) ||^{2}+|| \sum_{j=1}^{N}q_{j}g_{t}^{j} ||^{2}] \\&+ \frac{L^{2}}{2}\sum_{j=1}^{N}q_{j}|| \bar{w_{t}}-w_{t}^{j} ||^{2},
\end{align}
where $\textcircled{1}$ is from Lemma \ref{p_t_gamma} and $\textcircled{2}$ is from the common equation $ab=\frac{a^{2}+b^{2}-(a-b)^{2}}{2}$. $\textcircled{3}$ is due to the L-smooth properties of loss function and Jensen's inequality.

Next, the term $E[\sum_{j=1}^{N}q_{j}|| \bar{w}_{t}-w_{t}^{j} ||^{2}]$ therein will be estimated. In FL, model aggregation occurs in every $E_{l}$ steps, with its corresponding time slot is $t_{c}=\left \lfloor \frac{t}{E_{l}} \right \rfloor E_{l}$. From (\ref{gather_gd}), (\ref{gather_Sgd}), (\ref{gather_w}) and (\ref{update_w}), $\bar{w}_{t}$ and distributed $w_{t}^{j}$ in one local epoch are updated as follows.
\begin{align}
    &\bar{w}_{t_{c}}=\frac{1}{K_{\gamma}} \underset{j \in P_{t,\gamma}}{\sum}w_{t_{c}}^{j},\label{up_1}\\
    &w_{t}^{j}=\bar{w}_{t_{c}}-\sum_{k=t_{c}+1}^{t-1}\eta_{k}\widetilde{g}_{k}^{j},\label{up_2}\\
    &\bar{w}_{t}=\bar{w}_{t_{c}}-\frac{1}{K_{\gamma}} \underset{j \in P_{t,\gamma}}{\sum} \sum_{k=t_{c}+1}^{t-1} \eta_{k}\widetilde{g}_{k}^{j}.\label{up_3}
\end{align}
Note again that $\bar{w}_{t}$ is the virtually aggregated model weight for an arbitrary $t$. The real model aggregation and model broadcast only occurs at $t_{c}$.
Then due to (\ref{up_1}), (\ref{up_2}) and (\ref{up_3}), there is
\begin{align}\label{w_tup}
    &E[ \frac{1}{K_{\gamma}}\underset{j \in P_{t,\gamma}}{\sum}|| \bar{w}_{t}-w_{t}^{j} ||^{2} ] \notag\\& = E [\frac{1}{K_{\gamma}}\underset{j \in P_{t,\gamma}}{\sum} || \bar{w}_{t_{c}}- \frac{1}{K_{\gamma}}\underset{l \in P_{t,\gamma}}{\sum} \sum_{k=t_{c}+1}^{t-1} \eta_{k}\widetilde{g}_{k}^{l}-\bar{w}_{t_{c}}\notag\\&+\sum_{k=t_{c}+1}^{t-1}\eta_{k}\widetilde{g}_{k}^{j} ||^{2} ]=\notag\\
    &E [\frac{1}{K_{\gamma}}\underset{j \in P_{t,\gamma}}{\sum} || \sum_{k=1}^{r}\eta_{t_{c}+k}\widetilde{g}_{t_{c}+k}^{j}-\frac{1}{K^{\gamma}}\underset{l \in P_{t,\gamma}}{\sum} \sum_{k=1}^{r} \eta_{t_{c}+k}\widetilde{g}_{t_{c}+k}^{l} ||^{2}  ]\notag\\
    &\overset{\textcircled{1}}{=}E [\frac{1}{K_{\gamma}}\underset{j \in P_{t,\gamma}}{\sum} ||\sum_{k=1}^{r}\eta_{t_{c}+k}\widetilde{g}_{t_{c}+k}^{j}||^{2}] - E[|| \frac{1}{K^{\gamma}}\underset{j \in P_{t,\gamma}}{\sum}\notag\\& \sum_{k=1}^{r} \eta_{t_{c}+k}\widetilde{g}_{t_{c}+k}^{j} ||^{2}],
\end{align}
where $\textcircled{1}$ is due to $E|| x-E(x) ||^{2}=E|| x ||^{2}- [ E(x) ]^{2}$. The super index $l$ is applied to differ from $j$, they both represent mark for a client.

In (\ref{w_tup}), $\sum_{k=1}^{r}\eta_{t_{c}+k}\widetilde{g}_{t_{c}+k}^{j}$ is the accumulated gradient, which is also a form of the local gradient $\bigtriangledown f_{j}(w)$. From (\ref{lemma_summation}) in Lemma \ref{p_t_gamma}, $E_{P_{t,\gamma}}[\frac{1}{K_{\gamma}}\underset{j\in P_{t,\gamma}}{\sum}x_{j}] = \sum_{j=1}^{N}q_{j}x_{j}$. Then $E [\frac{1}{K_{\gamma}}\underset{j \in P_{t,\gamma}}{\sum} ||\sum_{k=1}^{r}\eta_{t_{c}+k}\widetilde{g}_{t_{c}+k}^{j}||^{2}]$ and $E[|| \frac{1}{K^{\gamma}}\underset{j \in P_{t,\gamma}}{\sum} \sum_{k=1}^{r} \eta_{t_{c}+k}\widetilde{g}_{t_{c}+k}^{j} ||^{2}]$ can be separately represented as $\sum_{j=1}^{N}q_{j}|| \bigtriangledown f_{j}(w) ||^{2}$ and $|| \sum_{j=1}^{N}q_{j}\bigtriangledown f_{j}(w) ||^{2}$. From the non-i.i.d. metric in Definition \ref{define_noniid}, there is
\begin{align}\label{up_gradient}
E [\frac{1}{K_{\gamma}}\underset{j \in P_{t,\gamma}}{\sum}& ||\sum_{k=1}^{r}\eta_{t_{c}+k}\widetilde{g}_{t_{c}+k}^{j}||^{2}] \notag\\& \leq \lambda E[|| \frac{1}{K^{\gamma}}\underset{j \in P_{t,\gamma}}{\sum} \sum_{k=1}^{r} \eta_{t_{c}+k}\widetilde{g}_{t_{c}+k}^{j} ||^{2}].
\end{align}
Therefore, (\ref{w_tup}) leads to
\begin{align}\label{up_noniid}
E[ \frac{1}{K_{\gamma}}&\underset{j \in P_{t,\gamma}}{\sum}|| \bar{w}_{t}-w_{t}^{j} ||^{2} ]\notag\\& \leq (\lambda-1)E[|| \frac{1}{K^{\gamma}}\underset{j \in P_{t,\gamma}}{\sum} \sum_{k=1}^{r} \eta_{t_{c}+k}\widetilde{g}_{t_{c}+k}^{j} ||^{2}].
\end{align}

Then from $E|| x ||^{2}=E|| x-E(x) ||^{2}+ [ E(x) ]^{2}$ and $E [ \widetilde{g}_{t}^{j}  ]=g_{t}^{j}$, (\ref{up_noniid}) further leads to

  \begin{align}\label{w_t1}
    &E[ \frac{1}{K_{\gamma}}\underset{j \in P_{t,\gamma}}{\sum}|| \bar{w}_{t}-w_{t}^{j} ||^{2} ] \leq\notag\\
    &(\lambda-1)E [ || \frac{1}{K_{\gamma}}\underset{j \in P_{t,\gamma}}{\sum} \sum_{k=1}^{r} \eta_{t_{c}+k}\widetilde{g}_{t_{c}+k}^{j} \notag\\&- \frac{1}{K_{\gamma}}\underset{j \in P_{t,\gamma}}{\sum} \sum_{k=1}^{r} \eta_{t_{c}+k}g_{t_{c}+k}^{j} ||^{2} + || \frac{1}{K_{\gamma}}\underset{j \in P_{t,\gamma}}{\sum} \sum_{k=1}^{r} \eta_{t_{c}+k}g_{t_{c}+k}^{j} ||^{2} ]\notag\\
    &=(\lambda-1)E[ || \frac{1}{K_{\gamma}}\underset{j \in P_{t,\gamma}}{\sum} \sum_{k=1}^{r} \eta_{t_{c}+k} ( \widetilde{g}_{t_{c}+k}^{j}-g_{t_{c}+k}^{j}  ) ||^{2} \notag\\&+ || \frac{1}{K_{\gamma}}\underset{j \in P_{t,\gamma}}{\sum} \sum_{k=1}^{r} \eta_{t_{c}+k}g_{t_{c}+k}^{j} ||^{2} ]\notag\\
    &\overset{\textcircled{1}}{=}(\lambda-1)E[ \frac{1}{K_{\gamma }^{2}} \underset{j \in P_{t,\gamma}}{\sum} \sum_{k=1}^{r} \eta_{t_{c}+k}^{2}|| \widetilde{g}_{t_{c}+k}^{j}-g_{t_{c}+k}^{j} ||^{2} \notag\\
    &+|| \frac{1}{K_{\gamma}}\underset{j \in P_{t,\gamma}}{\sum} \sum_{k=1}^{r}\eta_{t_{c}+k}g_{t_{c}+k}^{j} ||^{2} ],
  \end{align}
where $\textcircled{1}$ holds because $E_{\xi}(\widetilde{g}_{t_{c}+k}^{j}-g_{t_{c}+k}^{j})=0$ and the gradient errors are independent among clients.

Then from $|| \frac{1}{m}\sum_{j=1}^{m}x_{i} ||^{2} \leq \frac{1}{m}\sum_{j=1}^{m}|| x_{i} ||^{2}$ (Jensen's inequality), Eq. (\ref{w_t1}) can be bounded as
\begin{align}\label{w_t2}
&\overset{\textcircled{1}}{\leq}(\lambda-1)E[ \frac{1}{K_{\gamma}^{2}}\underset{j \in P_{t,\gamma}}{\sum} \sum_{k=1}^{r}\eta_{t_{c}+k}^{2}|| \widetilde{g}_{t_{c}+k}^{j}-g_{t_{c}+k}^{j} ||^{2} \notag\\
&+\frac{r}{K_{\gamma }}\underset{j \in P_{t,\gamma}}{\sum} \sum_{k=1}^{r}\eta_{t_{c}+k}^{2}|| g_{t_{c}+k}^{j} ||^{2} ].
\end{align}

Then under Assumption \ref{gradient_variance}, considering expectation $E_{\xi}[.]$, (\ref{w_t2}) can be bounded as follows.
\begin{align}\label{w_t3}
&E_{\xi}[\frac{1}{K_{\gamma}}\underset{j \in P_{t,\gamma}}{\sum}|| \bar{w}_{t}-w_{t}^{j} ||^{2}]\notag\\
&\leq  (\lambda-1) [\frac{1}{K_{\gamma}^{2}}\underset{j \in P_{t,\gamma}}{\sum} \sum_{k=1}^{r}\eta_{t_{c}+k}^{2} l_{j}(\frac{C_{1}}{D}|| g_{t_{c}+k}^{j} ||^{2}+\frac{\sigma^{2}}{D} ) \notag\\
&+\frac{r}{K_{\gamma}}\underset{j \in P_{t,\gamma}}{\sum} \sum_{k=1}^{r}\eta_{t_{c}+k}^{2}|| g_{t_{c}+k}^{j} ||^{2} ].
\end{align}

Then taking expectation $E_{P_{t,\gamma}}[.]$ on both sides of (\ref{w_t3}) and applying Lemma \ref{p_t_gamma}, we have
\begin{align}\label{w_t5}
    &\sum_{j=1}^{N}q_{j}|| \bar{w}_{t}-w_{t}^{j} ||^{2} \leq \notag\\
    &(\lambda-1)[\frac{1}{K(1-\gamma)}\sum_{j=1}^{N}q_{j} \sum_{k=1}^{r}\eta_{t_{c}+k}^{2}l_{j}(\frac{C_{1}}{D}|| g_{t_{c}+k}^{j} ||^{2}+\frac{\sigma^{2}}{D} ) \notag\\
    &+r\sum_{j=1}^{N}q_{j} \sum_{k=1}^{r}\eta_{t_{c}+k}^{2}|| g_{t_{c}+k}^{j} ||^{2} ] \notag\\
    &=(\lambda-1)[\frac{C_{1}}{D}\frac{1}{ K(1-\gamma)} \sum_{j=1}^{N}q_{j} \sum_{k=1}^{r}\eta_{t_{c}+k}^{2}l_{j}|| g_{t_{c}+k}^{j} ||^{2}+ \notag\\
    &\frac{1}{K(1-\gamma)}\frac{\sigma^{2}}{D}\sum_{j=1}^{N}q_{j}l_{j}\sum_{k=1}^{r}\eta_{t_{c}+k}^{2} + r\sum_{j=1}^{N}q_{j} \sum_{k=1}^{r}\eta_{t_{c}+k}^{2}|| g_{t_{c}+k}^{j} ||^{2} ]\notag\\
    &\overset{\textcircled{1}}{\leq}\frac{\lambda-1}{K(1-\gamma)}\frac{C_{1}}{D} \sum_{j=1}^{N}q_{j}l_{j} \sum_{k=1}^{r}\eta_{t_{c}+k}^{2}|| g_{t_{c}+k}^{j} ||^{2}+\frac{\sigma^{2}}{D}\sum_{j=1}^{N}q_{j}l_{j}\notag\\
    &\frac{\lambda-1}{K(1-\gamma)}\sum_{k=1}^{r}\eta_{t_{c}+k}^{2}+(\lambda-1)E_{l}\sum_{j=1}^{N}q_{j} \sum_{k=1}^{r}\eta_{t_{c}+k}^{2}|| g_{t_{c}+k}^{j} ||^{2},
\end{align}
where $\textcircled{1}$ comes from the fact that $r \leq E_{l}$ in FL training. Then under (\ref{re_weight}), (\ref{w_t5}) is further transformed as
\begin{align}\label{w_t6}
    &\sum_{j=1}^{N} q_{j}|| \bar{w}_{t}-w_{t}^{j} ||^{2}
    \leq \notag\\
    &=\frac{\lambda-1}{K(1-\gamma)}\frac{C_{1}}{D}\sum_{k=1}^{r}\eta_{t_{c}+k}^{2}\sum_{j=1}^{N}q_{j}l_{j}|| g_{t_{c}+k}^{j} ||^{2}+\frac{\sigma^{2}}{D} \notag\\
    &\frac{\lambda-1}{K(1-\gamma)}\sum_{k=1}^{r}\eta_{t_{c}+k}^{2}
    +(\lambda-1)E_{l} \sum_{k=1}^{r}\eta_{t_{c}+k}^{2}\sum_{j=1}^{N} q_{j}|| g_{t_{c}+k}^{j} ||^{2}\notag\\
    &=\frac{\lambda-1}{K(1-\gamma)}\frac{C_{1}}{D}\sum_{k=t_{c}+1}^{t-1}\eta_{k}^{2}\sum_{j=1}^{N}q_{j}l_{j}|| g_{k}^{j} ||^{2}+ \notag\\
    &\frac{\sigma^{2}}{D}\frac{\lambda-1}{K(1-\gamma)}\sum_{k=t_{c}+1}^{t-1}\eta_{k}^{2}
    +(\lambda-1)E_{l} \sum_{k=t_{c}+1}^{t-1}\eta_{k}^{2}\sum_{j=1}^{N} q_{j}|| g_{k}^{j} ||^{2},
\end{align}
where $\{q_{j}\}$ and $\{q_{j}l_{l}\}$ can both be viewed as set of weights with $\sum_{j=1}^{N}q_{j}=1$ and $\sum_{j=1}^{N}q_{j}l_{j}=1$. Then consider metrics of non-i.i.d. data set in Definition \ref{define_noniid}, (\ref{w_t6}) leads to
\begin{align}\label{w_t8}
&\sum_{j=1}^{N} q_{j}|| \bar{w}_{t}-w_{t}^{j} ||^{2}
    \leq \notag\\
&\frac{\lambda-1}{K(1-\gamma)}\frac{C_{1}}{D}\lambda\sum_{k=t_{c}+1}^{t-1}\eta_{k}^{2}||\sum_{j=1}^{N}q_{j}l_{j} g_{k}^{j} ||^{2}+ \notag\\
    &\frac{\sigma^{2}}{D}\frac{\lambda-1}{K(1-\gamma)}\sum_{k=t_{c}+1}^{t-1}\eta_{k}^{2}
    +(\lambda-1)E_{l}\lambda \sum_{k=t_{c}+1}^{t-1}\eta_{k}^{2}||\sum_{j=1}^{N} q_{j} g_{k}^{j} ||^{2}.
\end{align}
Considering the limited diminishing speed of $\eta_{t}$, it is reasonable to have $\eta_{k}^{2}\leq \eta_{t} (t_{c}+1 \leq k \leq t)$. Besides, according to the definition of $t_{c}=\left \lfloor \frac{t}{E_{l}} \right \rfloor E_{l}$, the summation of $k$ can be upper-bounded with a range from $t_{c}+1$ to $t_{c}+E_{l}$. Then, (\ref{w_t8}) can be transformed as follows.
\begin{align}\label{w_t9}
&\sum_{j=1}^{N} q_{j}|| \bar{w}_{t}-w_{t}^{j} ||^{2}
    \leq \notag\\
&\frac{\lambda-1}{K(1-\gamma)}\frac{C_{1}}{D}\eta_{t}\lambda\sum_{k=t_{c}+1}^{t_{c}+E_{l}}||\sum_{j=1}^{N}q_{j}l_{j} g_{k}^{j} ||^{2}+ \notag\\
    &\frac{E_{l}\eta_{t}\sigma^{2}}{D}\frac{\lambda-1}{K(1-\gamma)}
    +(\lambda-1)E_{l}\lambda \eta_{t}\sum_{k=t_{c}+1}^{t_{c}+E_{l}}||\sum_{j=1}^{N} q_{j} g_{k}^{j} ||^{2}.
\end{align}
By combining (\ref{w_t9}) with (\ref{et_cross}), the proof is completed.

\section{Lemma \ref{math_induction}}\label{proof_induction}
\begin{lemma}\label{math_induction}
If $E[f(\bar{w}_{t+1})]-f^{*} \leq \eta_{t}^{2}M + (1-\mu \eta_{t})[E[f(\bar{w}_{t})]-f^{*}]$ holds and $\eta_{t} \propto O(\frac{1}{t})$, then the training loss converges as
\begin{align}\label{conv_bd}
E[f(\bar{w}_{t})]-f^{*} = \frac{1}{t}\textbf{\rm max}\{\frac{4}{\mu^{2}}M \,, 2L\lambda\bigtriangleup_{1}\},
\end{align}
where $\bigtriangleup_{1}$ is the initial loss. $L$ is from Assumption \ref{L_smooth} and $\lambda$ is the non-i.i.d. metric in Definition \ref{define_noniid}.
\end{lemma}

\textbf{Proof: }

Assume that the training loss function of FL satisfies the following form.
\begin{align}\label{loss_form}
E[f(\bar{w}_{t+1})]-f^{*} \leq \eta_{t}^{2}M + (1-\mu \eta_{t})[E[f(\bar{w}_{t})]-f^{*}],
\end{align}
where $M$ is short for a closed-form math expression.
We denote $\bigtriangleup_{t}=E[f(\bar{w}_{t})]-f^{*}$, then (\ref{loss_form}) can be simplified as follows.
\begin{equation}\label{conv_1}
  \bigtriangleup_{t+1}\leq (1-\mu \eta_{t})\bigtriangleup_{t} + \eta_{t}^{2}M.
\end{equation}

Generally, the learning rate can be represented as $\eta_{t}=\frac{v}{(t+\beta)^{\alpha}}$, where $v$ and $\beta$ are parameters from the initial $\eta_{0}$ and $\alpha$ represents its diminishing speed. If $\alpha \leq 1$, $\bigtriangleup_{t}$ is proved to be bounded in the form as follows.
\begin{equation}\label{f_conv}
  \bigtriangleup_{t} \leq \frac{X}{(t+\beta)^{\alpha}},
\end{equation}
where $X$ is short for a closed-form expression which will be defined in the following proof.

The proof is supported by mathematical induction method. If $\bigtriangleup_{t} \leq \frac{X}{(t+\beta)^{\alpha}}$ holds, we will prove that $\bigtriangleup_{t+1} \leq \frac{X}{(t+1+\beta)^{\alpha}}$ holds only for $\alpha \leq 1$.
By taking $\eta_{t}=\frac{v}{(t+\beta)^{\alpha}}$ and $\bigtriangleup_{t} \leq \frac{X}{(t+\beta)^{\alpha}}$ into (\ref{conv_1}), we get
  \begin{align}\label{conv_2}
    \bigtriangleup_{t+1} &\leq [1-\frac{\mu v}{(t+\beta)^{\alpha}}]\bigtriangleup_{t}+[\frac{v}{(t+\beta)^{\alpha}}]^{2} M \notag\\
    &\leq [1-\frac{\mu v}{(t+\beta)^{\alpha}}]\frac{X}{(t+\beta)^{\alpha}}+[\frac{v}{(t+\beta)^{\alpha}}]^{2} M \notag\\
    &=\frac{(t+\beta)^{\alpha}-\mu v}{(t+\beta)^{2\alpha}} X + \frac{v^{2}}{(t+\beta)^{2\alpha}} M\notag\\
    &=\frac{(t+\beta)^{\alpha}-1}{(t+\beta)^{2\alpha}} X + [\frac{v^{2}M}{(t+\beta)^{2\alpha}}-\frac{\mu v-1}{(t+\beta)^{2\alpha}}X].
  \end{align}
If parameter $X$ satisfies $X\geq \frac{v^{2}M}{\mu v -1}$, where $\mu v \geq 1$, there is $\frac{v^{2}M}{(t+\beta)^{2\alpha}}-\frac{\mu v-1}{(t+\beta)^{2\alpha}}X \leq 0$. Then, the following bound holds.
\begin{equation}\label{conv3}
  \bigtriangleup_{t+1} \leq \frac{(t+\beta)^{\alpha}-1}{(t+\beta)^{2\alpha}} X.
\end{equation}
Therefore, the problem reduces to proving $\frac{(t+\beta)^{\alpha}-1}{(t+\beta)^{2\alpha}} \leq \frac{1}{(t+1+\beta)^{\alpha}}$. Denote $(t+\beta)=y$, there is $y\geq 1$. Then
\begin{align}\label{prooft_1}
&\frac{(t+\beta)^{\alpha}-1}{(t+\beta)^{2\alpha}} \leq \frac{1}{(t+1+\beta)^{\alpha}}\notag\\
&\Leftrightarrow y^{2\alpha}\geq (y^{\alpha}-1)(y+1)^{\alpha} \notag\\
&\Leftrightarrow 1 \geq (1-\frac{1}{y^{\alpha}})(1+\frac{1}{y})^{\alpha}.
\end{align}

Define function $h(m)=(1-m^{\alpha})(1+m)^{\alpha}$, where $0 \leq m \leq 1$. Proof of (\ref{prooft_1}) corresponds to proving $h(m)\leq 1$. The derivative of $h(m)$ is
\begin{equation}\label{deri_h}
  h^{'}(m)=\alpha(1+m)^{\alpha-1}(1-2m^{\alpha}-m^{\alpha-1}).
\end{equation}
From definition of $y$, $m=\frac{1}{t+\beta}$.
Now let us observe (\ref{deri_h}). Firstly, $h(0)=1$ holds. If $\alpha>1$, $h^{'}(m)$ is definitely positive as $m$ approaches $0$, which means $h(m)\leq 1$ does not hold for $0 \leq m \leq 1$. If $\alpha \leq 1$, $h^{'}(m)$ remains negative for $m \in [0,1]$, which ensures that bound (\ref{prooft_1}) holds for $t\geq 1$. Therefore, it leads to the conclusion that $\bigtriangleup_{t} \leq \frac{X}{(t+\beta)^{\alpha}}$ leads to $\bigtriangleup_{t+1} \leq \frac{X}{(t+1+\beta)^{\alpha}}$ for $\alpha \leq 1$. Considering initial loss, let $\bigtriangleup_{1} \leq \frac{X}{(1+\beta)^{\alpha}}$, parameter $X$ should be
\begin{equation}\label{set_X}
  X=\textbf{\rm max}\left \{ \frac{v^{2}M}{\mu v -1}, (\beta+1)^{\alpha}\bigtriangleup_{1}\ \right \}.
\end{equation}
Set $v=\frac{2}{\mu}$, $\beta+1=2L\lambda$ and $\alpha=1$, it gets to
  \begin{align}\label{final_bound1}
    E[f(\bar{w}_{t})]-f^{*} \leq \frac{1}{t+\beta} \textbf{\rm max}\{\frac{4}{\mu^{2}}M \,, 2L\lambda\bigtriangleup_{1}\}.
  \end{align}
Omitting the unimportant bias $\beta$, the proof is completed with results in Lemma \ref{math_induction}.

\section{Proof of Theorem \ref{estimate_G_epsilon}}\label{proof_G}

Under (\ref{update_w}) and Assumption \ref{L_smooth}, the training loss function satisfies
\begin{equation}\label{basis_loss}
\begin{aligned}
    E[f(\bar{w}_{t+1})-f(\bar{w_{t}})] \leq -\eta_{t}E[<\bigtriangledown f(\bar{w_{t}}), \widetilde{g}_{t}>]+\frac{\eta_{t}^{2}L}{2}E|| \widetilde{g}_{t} ||^{2}.
\end{aligned}
\end{equation}

From Lemma \ref{g_t} in Appendix \ref{proof_g_t}, $E[ || \widetilde{g}_{t} ||^{2}  ]$ is bounded as
\begin{align}
    &E[ || \widetilde{g}_{t} ||^{2}  ] \leq \notag\\
    &\frac{\lambda C_{1}}{DK(1-\gamma)} \sum_{j=1}^{N}|| q_{j}l_{j}g_{t}^{j} ||^{2}
    +\frac{1}{K(1-\gamma)}\frac{\sigma^{2}}{D}+\lambda \sum_{j=1}^{N}|| q_{j}g_{t}^{j} ||^{2}.
\end{align}
From Lemma \ref{cross_bd} in Appendix \ref{proof_cross}, $-\eta_{t}E\left [ <\bigtriangledown f(\bar{w_{t}}), \widetilde{g}_{t}> \right ]$ is upper-bounded as follows.
\begin{align}\label{cross_bound}
    &-\eta_{t}E\left [ <\bigtriangledown f(\bar{w_{t}}), \widetilde{g}_{t}> \right ]\leq
    -\frac{1}{2}|| \bigtriangledown f(\bar{w_{t}}) ||^{2}-\frac{1}{2}|| \sum_{j=1}^{N}q_{j}g_{t}^{j} ||^{2} \notag\\
    &+\eta_{t}^{2}L^{2}\lambda\frac{\lambda-1}{K(1-\gamma)}\frac{C_{1}}{2D}\sum_{k=t_{c}+1}^{t_{c}+E_{l}}||\sum_{j=1}^{N}q_{j}l_{j} g_{k}^{j} ||^{2}+ \notag\\
    &\frac{E_{l}\eta_{t}^{2}L^{2}\sigma^{2}}{2D}\frac{\lambda-1}{K(1-\gamma)}
    +\frac{\lambda-1}{2}E_{l}\lambda L^{2} \eta_{t}^{2}\sum_{k=t_{c}+1}^{t_{c}+E_{l}}||\sum_{j=1}^{N} q_{j} g_{k}^{j} ||^{2}.
\end{align}

By properties of $\mu$-P-L condition in Assumption \ref{P_L_condition}, term $-\frac{1}{2}|| \bigtriangledown f(\bar{w_{t}}) ||^{2}$ in Lemma \ref{cross_bd} can be substituted by $-\mu[f(x)-f^{*}]$.
Then applying Lemma \ref{g_t} and \ref{cross_bd} to (\ref{basis_loss}), it gets to
\begin{align}\label{loss_bound1}
&E[f(\bar{w}_{t+1})]-f^{*} \leq (1-\mu \eta_{t})[E[f(\bar{w}_{t})]-f^{*}]\notag\\
    &-\frac{\eta_{t}}{2}|| \sum_{j=1}^{N}q_{j}g_{t}^{j} ||^{2}+\frac{\eta_{t}^{2}}{2}L\lambda|| \sum_{j=1}^{N}q_{j}g_{t}^{j} ||^{2} +\frac{\lambda-1}{2}E_{l}\lambda L^{2} \eta_{t}^{2}\notag\\ &\sum_{k=t_{c}+1}^{t_{c}+E_{l}}||\sum_{j=1}^{N} q_{j} g_{k}^{j} ||^{2}
    +\frac{\eta_{t}^{2}}{2}L\lambda (\frac{C_{1}}{D K(1-\gamma)})|| \sum_{j=1}^{N}q_{j}l_{j}g_{t}^{j} ||^{2}\notag\\
    &+E_{l}\eta_{t}^{2}L^{2}\lambda\frac{\lambda-1}{K(1-\gamma)}\frac{C_{1}}{2D}\sum_{k=t_{c}+1}^{t_{c}+E_{l}}|| \sum_{j=1}^{N}q_{j}l_{j}g_{k}^{j} ||^{2}\notag\\
    &+\eta_{t}^{2}L^{2}\frac{\lambda-1}{K(1-\gamma)}\frac{E_{l}\sigma^{2}}{2D}
    +\frac{\eta_{t}^{2}L}{2}\frac{1}{K(1-\gamma)}\frac{\sigma^{2}}{D}.
\end{align}
Set the learning rate as $\eta_{t} \leq \frac{1}{L\lambda}$, it leads to $\frac{\eta_{t}^{2}}{2}L\lambda || \sum_{j=1}^{N}q_{j}g_{t}^{j} ||^{2} - \frac{\eta_{t}}{2}|| \sum_{j=1}^{N}q_{j}g_{t}^{j} ||^{2} \leq 0$. Then this term in the above upper-bound can be removed. Therefore, by simple transformation, (\ref{loss_bound1}) is equivalent to
\begin{align}\label{loss_bound2}
&E[f(\bar{w}_{t+1})]-f^{*} \leq (1-\mu \eta_{t})[E[f(\bar{w}_{t})]-f^{*}]+\notag\\
&\eta_{t}^{2}L^{2}\lambda\frac{\lambda-1}{K(1-\gamma)}\frac{1}{2D}\sum_{k=t_{c}+1}^{t_{c}+E_{l}}(C_{1}|| \sum_{j=1}^{N}q_{j}l_{j}g_{k}^{j} ||^{2}+\frac{\sigma^{2}}{\lambda})\notag\\
&+\frac{\eta_{t}^{2}}{2}L\lambda\frac{1}{DK(1-\gamma)} (C_{1}|| \sum_{j=1}^{N}q_{j}l_{j}g_{k}^{j} ||^{2}+\frac{\sigma^{2}}{\lambda})\notag\\
    &+\eta_{t}^{2}L^{2}\lambda \frac{(\lambda-1)E_{l}}{2C_{1}} C_{1}\sum_{k=t_{c}+1}^{t_{c}+E_{l}}|| \sum_{j=1}^{N}q_{j}g_{k}^{j} ||^{2}.
\end{align}
Under (\ref{re_weight}) and Assumption \ref{gradient_bound}, term $(C_{1}|| \sum_{j=1}^{N}q_{j}l_{j}g_{k}^{j} ||^{2}+\frac{\sigma^{2}}{\lambda})$ is upper-bounded by $G^{2}$. By definitions in (\ref{phi0}), it has $C_{1}|| \sum_{j=1}^{N}q_{j}g_{k}^{j} ||^{2} \leq \frac{G^{2}}{\phi_{0}}$. Then it leads to
\begin{align}\label{loss_bound3}
&E[f(\bar{w}_{t+1})]-f^{*} \leq (1-\mu \eta_{t})[E[f(\bar{w}_{t})]-f^{*}] \notag\\
    &+\eta_{t}^{2}E_{l}G^{2}L^{2}\lambda[\frac{1}{2LKD(1-\gamma)}\frac{1}{E_{l}}+\frac{\lambda-1}{K(1-\gamma)}\frac{1}{2D}]\notag\\
    &+\eta_{t}^{2}E_{l}G^{2}L^{2}\lambda \frac{(\lambda-1)E_{l}}{2C_{1}\phi_{0}}.
\end{align}
Based on (\ref{loss_bound3}), a shorthand notation is defined as $M=E_{l}G^{2}L^{2}\lambda[\frac{(\lambda-1)E_{l}}{2C_{1}\phi_{0}}+\frac{1}{2LKD(1-\gamma)}\frac{1}{E_{l}}+\frac{\lambda-1}{K(1-\gamma)})\frac{1}{2D}]$, (\ref{loss_bound3}) comes to
\begin{align}\label{short_bound}
E[f(\bar{w}_{t+1})]-f^{*} \leq (1-\mu \eta_{t})[E[f(\bar{w}_{t})]-f^{*}]+\eta_{t}^{2}M.
\end{align}
Then from Lemma \ref{math_induction} in Appendix \ref{proof_induction}, the training loss converges as
\begin{align}\label{conv_bd}
E[f(\bar{w}_{t})]-f^{*} = \frac{1}{t}\textbf{\rm max}\{\frac{4}{\mu^{2}}M \,, 2L\lambda\bigtriangleup_{1}\}.
\end{align}
In (\ref{conv_bd}), $2L\lambda\bigtriangleup_{1}$ represents the initial training loss.
By properties of L-smooth, $\bigtriangleup_{1} \leq \frac{L}{2}|| w_{0}-w^{*} ||^{2}$. By $\mu$-strongly convexity, $\mu||w_{0}-w^{*}|| \leq ||\bigtriangledown f(w_{0})-\bigtriangledown f(w^{*})||$. Since $\bigtriangledown f(w^{*})=0$, it leads to $||w_{0}-w^{*}||^{2}\leq \frac{1}{\mu^{2}}\left \| \bigtriangledown f(w_{0}) \right \|^{2}$. Suppose $\left \| \bigtriangledown f(w_{0}) \right \|^{2} \leq f_{0}G^{2}$, then $\left \| w_{0}-w^{*} \right \|^{2} \leq \frac{f_{0}}{\mu^{2}}G^{2}$. Therefore, $2L\lambda \bigtriangleup_{1} \leq \frac{L^{2}f_{0}\lambda G^{2}}{\mu^{2}}$. Taking in $M=E_{l}G^{2}L^{2}\lambda[\frac{(\lambda-1)E_{l}}{2C_{1}\phi_{0}}+\frac{1}{2LKD(1-\gamma)}\frac{1}{E_{l}}+\frac{\lambda-1}{K(1-\gamma)}\frac{1}{2D}]$
together with Lemma \ref{math_induction} and (\ref{short_bound}), the training loss converges as
  \begin{align}\label{final_bound2}
    &E[f(\bar{w}_{t})]-f^{*} \leq \notag\\ &\frac{1}{t} \textbf{\rm max}\{\frac{4}{\mu^{2}}L^{2}\lambda E_{l} G^{2}[\frac{(\lambda-1)E_{l}}{2C_{1}\phi_{0}}+\frac{1}{2LKD(1-\gamma)}\frac{1}{E_{l}}\notag\\ &+\frac{\lambda-1}{K(1-\gamma)}\frac{1}{2D}]\,, \frac{L^{2}G^{2}\lambda f_{0}}{\mu^{2}}\} \notag\\
    &=\frac{1}{t}\frac{4L^{2}G^{2}\lambda}{\mu^{2}} \textbf{\rm max}\{ E_{l}[\frac{(\lambda-1)E_{l}}{2C_{1}\phi_{0}}+\frac{1}{2LKD(1-\gamma)}\frac{1}{E_{l}}\notag\\
    &+\frac{\lambda-1}{K(1-\gamma)}\frac{1}{2D}]\,, \frac{1}{4}f_{0} \}.
  \end{align}
Term $\frac{1}{4}f_{0}$ represents error from initial loss.
Adding the two elements together, the loss function converges as follows.
\begin{align}\label{training_loss_es}
&E[f(\bar{w}_{t+1})]-f^{*} \leq \frac{1}{t}\frac{4E_{l}L^{2}G^{2}\lambda}{\mu^{2}} \notag\\ &[\frac{(\lambda-1)E_{l}}{2C_{1}\phi_{0}}+\frac{1}{2LKD(1-\gamma)}\frac{1}{E_{l}}+\frac{1}{4}\frac{f_{0}}{E_{l}}+\frac{\lambda-1}{K(1-\gamma)}\frac{1}{2D}].
\end{align}

Note that $\frac{1}{2LKD(1-\gamma)}\frac{1}{E_{l}}$ is relatively small compared with $\frac{1}{4}\frac{f_{0}}{E_{l}}$, $\frac{(\lambda-1)E_{l}}{2C_{1}\phi_{0}}$ and $\frac{\lambda-1}{K(1-\gamma)}\frac{1}{2D}$. Therefore, it can just be omitted for simplicity, which will not cause large effects to the tendency of training parameters. Then we have
\begin{align}\label{training_loss_bound}
E[f&(\bar{w}_{t+1})]-f^{*}\notag\\& \leq \frac{1}{t}\frac{4E_{l}L^{2}G^{2}\lambda}{\mu^{2}} [\frac{(\lambda-1)E_{l}}{2C_{1}\phi_{0}}+\frac{1}{4}\frac{f_{0}}{E_{l}}+\frac{\lambda-1}{K(1-\gamma)}\frac{1}{2D}].
\end{align}

Define $t_{\epsilon}$ as the time slot where $E[f(\bar{w}_{t_{\epsilon}+1})]-f^{*}$ reaches $\epsilon$, the corresponding global training epoch $G_{\epsilon}=\frac{t_{\epsilon}}{E_{l}}$. From (\ref{training_loss_bound}), it is straightforward to get
\begin{align}\label{conv_global}
G_{\epsilon}= \frac{1}{\epsilon}\frac{4L^{2}G^{2}\lambda}{\mu^{2}}[ \frac{\lambda-1}{K(1-\gamma)}\frac{1}{2D} +\frac{(\lambda-1)E_{l}}{2C_{1}\phi_{0}}+\frac{1}{4}\frac{f_{0}}{E_{l}} ].
\end{align}

Here the proof is completed.

\section{Proof of Theorem \ref{K_result}}\label{proof_K_result}
Firstly observing results in Theorem \ref{estimate_G_epsilon}, the optimal local epoch is $E_{l}^{*}=\sqrt{\frac{C_{1}\phi_{0}f_{0}}{2(\lambda-1)}}$. Putting $E_{l}^{*}$ in to (\ref{G_theorem}), an estimation of $G_{\epsilon}$ can be obtained as follows.
\begin{align}\label{estimate_G1}
G_{\epsilon}=O(\sqrt{\frac{(\lambda-1)f_{0}}{8C_{1}\phi_{0}}}+\frac{\lambda-1}{2D(1-\gamma)}\frac{1}{K}).
\end{align}

From (\ref{u_Cu}) and (\ref{u_Cn}), cost in one global epoch is
\begin{align}\label{C_g}
C_{g}&=E_{l}(C_{n,0}+T_{d}/E_{l})+KC_{u,0}.
\end{align}
Combining (\ref{estimate_G1}) and (\ref{C_g}), the overall training cost can be obtained as
\begin{align}\label{global_cost}
&C_{g}G_{\epsilon}=E_{l}(C_{n,0}+T_{d}/E_{l})\sqrt{\frac{(\lambda-1)f_{0}}{8C_{1}\phi_{0}}}+\frac{C_{u,0}}{1-\gamma}\frac{\lambda-1}{2D}\notag\\
&+KC_{u,0}\sqrt{\frac{(\lambda-1)f_{0}}{8C_{1}\phi_{0}}}+\frac{\lambda-1}{2D(1-\gamma)}\frac{E_{l}(C_{n,0}+T_{d}/E_{l})}{K}.
\end{align}
Therefore, the optimal $K^{*}$ to minimize $C_{g}G_{\epsilon}$ is
\begin{align}\label{proof_K}
K^{*}=\sqrt[4]{\frac{2(\lambda-1)C_{1}\phi_{0}}{f_{0}}} \sqrt{\frac{E_{l}}{D(1-\gamma)}} \sqrt{\frac{C_{n,0}+T_{d}/E_{l}}{C_{u,0}}}.
\end{align}
Here the proof is completed.

\section{Proof of Theorem \ref{t_aj}}\label{proof_t_aj}
From KKT conditions of the Lagrange function (\ref{L_function}), the following equations can be derived.
\begin{equation}\label{kkta}
\frac{\partial L}{\partial a_{j}}=R-\frac{z_{m}\beta_{j}}{Br_{j}^{0}}\frac{1}{a_{j}^{2}}=0,
\end{equation}
\begin{equation}\label{kktf}
\frac{\partial L}{\partial f_{j}}=\frac{2l_{0}}{K}E_{l}\alpha_{0}\kappa_{j}D_{j}f_{j}-\frac{\beta_{j}E_{l}\alpha_{0}D_{j}}{f_{j}^{2}}=0,
\end{equation}
\begin{equation}\label{kktr}
R(\underset{j\in P_{t}}{\sum}a_{j}-1)=0,
\end{equation}
From (\ref{kktf}) and (\ref{kkta}), we get
\begin{equation}\label{kkt1}
R=\frac{\beta_{j}z_{m}}{Br_{j}^{0}}\frac{1}{a_{j}^{2}},
\end{equation}
\begin{equation}\label{kkt2}
\beta_{j}=\frac{2l_{0}}{K}\kappa_{j}f_{j}^{3}.
\end{equation}
Combined with $\underset{j\in P_{t}}{\sum}a_{j}=1$ from (\ref{kktr}), (\ref{kkt1}) leads to
\begin{equation}
\underset{j\in P_{t}}{\sum} \sqrt{\frac{\beta_{j}z_{m}}{Br_{j}^{0}}}=\sqrt{R}=\sqrt{\frac{\beta_{j}z_{m}}{Br_{j}^{0}}}/a_{j}.
\end{equation}
Taking in (\ref{kkt2}), the results in Theorem \ref{t_aj} can be derived.

\ifCLASSOPTIONcaptionsoff
  \newpage
\fi

\end{document}